\documentclass[11pt]{article}
\input{header}

\begin{document}
\title{PriorBoost: An Adaptive Algorithm for Learning from Aggregate Responses}

\author[1,2]{Adel Javanmard}
\author[2]{Matthew Fahrbach}
\author[2]{Vahab Mirrokni}
\affil[1]{University of Southern California, \texttt{ajavanma@usc.edu}}
\affil[2]{Google Research, \texttt{\{fahrbach,mirrokni\}@google.com}}

\date{}

\maketitle

\begin{abstract}
This work studies algorithms for learning from aggregate responses.
We focus on the construction of aggregation sets (called \emph{bags} in the literature) for event-level loss functions.
We prove for linear regression and generalized linear models (GLMs) that the optimal bagging problem reduces to
one-dimensional size-constrained $k$-means clustering.
Further, we theoretically quantify the advantage of using curated bags over random bags.
We then propose the \texttt{PriorBoost} algorithm, which adaptively forms bags of samples
that are increasingly homogeneous with respect to (unobserved) individual responses
to improve model quality.
We study label differential privacy for aggregate learning,
and we also provide extensive experiments showing that \PriorBoost
regularly achieves optimal model quality for event-level predictions,
in stark contrast to non-adaptive algorithms.
\end{abstract}

\section{Introduction}
\label{sec:introduction}

In supervised learning, the learner is given a training dataset of $n$ i.i.d pairs $(\mat{x}_i, y_i)$, where 
$\mat{x}_i \in \R^d$ is a feature vector and
$y_i$ is the corresponding response.
Responses are real-valued for regression problems,
and belong to a finite discrete set for multi-class classification.
The fundamental problem in supervised learning is to (1) train a model with this data,
and (2) use this model to infer the response/label of unseen test instances.
However, in many practical applications (e.g., medical tests and elections),
the responses contain sensitive information,
but the features are far less sensitive (e.g., demographic information or zip codes/regions).
In such applications, there are valid concerns about revealing \emph{individual responses}
to the learning algorithm, even if it is a trusted party.  

A popular approach to mitigate this privacy concern in practice
is to let the learner access responses in an \emph{aggregate manner}.
In the framework of \emph{learning from aggregate responses} (LAR),
the learner is given access to a collection of unlabeled feature vectors called \emph{bags}
and an aggregate summary of the responses in each bag.
A widely used choice is the mean response or label proportions of each bag~\citep{yu2014learning}.
The learner then fits a model using the aggregate responses with the goal of
accurately \emph{predicting individual responses} on future data.

The problem of learning from aggregate responses (a.k.a.\ learning from label proportions in the context of classification) dates back to at least~\citet{wein1996pooled} in the context of group testing, a technique used in many different fields including medical diagnostics, population screening, and quality control. The idea is to combine multiple samples into a group and test them together rather than individually. This approach has been widely adopted in cases where testing resources are limited or the prevalence of the condition being tested for is low.
LAR has also been studied in other earlier work~\citep{DBLP:conf/uai/FreitasK05,musicant2007supervised,quadrianto2008estimating,rueping2010svm,NIPS2014_a8baa565} for settings where direct access to the individual responses is not possible (e.g., in political party elections where aggregate votes are only available at discrete district levels).

Recently, there has been a resurgence in the LAR framework primarily due to the rise of privacy concerns; see~\citep{scott2020learning,saket2022algorithms,zhang2022learning,busa2023easy,chen2023learning,brahmbhatt2023llp,javanmard2024loss} for a non-exhaustive list.
Specifically, if the aggregation bags are large enough and have no (or little) overlap,
revealing only the aggregate responses provides a layer of privacy protection,
often formalized in terms of \emph{$k$-anonymity}~\citep{sweeney2002k}.
Large tech companies have recently deployed aggregate learning frameworks, including
Apple's SKAdNetwork library~\citep{kollnig2022goodbye}
and
the Private Aggregation API in the Google Privacy Sandbox~\citep{geradin2020google}.
Aggregate responses can further be perturbed to provide label differential privacy~\citep{chaudhuri2011sample},
a popular notion of privacy that measures the leakage of personal label/response information,
which we discuss in detail in \Cref{sec:differential_privacy}.

In some applications, the bagging configurations are naturally determined by the problem at hand
(e.g., in the voting example above the bags are defined based on districts).
In other applications, however, the learner has the flexibility of curating bags of query samples to maximize model utility while complying with  privacy or legal constraints imposed by the data regulators.
Our work focuses on the problem of \emph{bag curation}
in the framework of learning from aggregate responses. 

\subsection{Problem statement}
We first describe the process of learning from aggregate responses, for a given collection of bags. Consider a partition of $n$ samples into $m$ non-overlapping bags, each of size at least $k$, for a prespecified $k$ (and hence $n\ge m k$). We focus on training a model by minimizing the following \emph{event-level loss}:
\begin{align}
    \label{eq:event0}
        \hth
        :=
        \argmin_{\th} \frac{1}{n} \sum_{\ell=1}^m \sum_{i\in B_\ell} \mathcal{L}(\overline{y}_{\ell},f_{\th}(\mat{x}_i))\,,
    \end{align}
    where $B_\ell$ is the set of samples in bag $\ell$
    and $\overline{y}_{\ell}$ is the mean response in bag $\ell$.
    In words, with this approach the model is learned by fitting individual predictions to the average response of its bag.
    
    The problem of bag curation is to find an optimal bagging configuration
    that maximizes model utility (in terms of minimizing estimation error),
    while satisfying the minimum bag size constraint $|B_\ell| \ge k$. Note that this
    min-size constraint implies $k$-anonymity in the sense of that any response in the (aggregate) dataset is shared by at least $k$ individuals. Larger values of $k$ offer higher protection of individual responses.  

\subsection{Overview of our approach and contributions}
This work focuses on event-level loss and the problem of bag curation.
To control privacy leakage, we require the bags to be non-overlapping
and of size at least $k$.
An important property of our mechanism is the following: 
The learner \emph{never sees an individual response}.
Conceptually, the learner always constructs a query of fresh samples
to send to an oracle,
who then returns the aggregate response.

Our key insight is to leverage available prior information about $\E[y \mid \bx]$
to construct better bags for the learner.
Such prior information can be based on domain knowledge,
models trained on public data,
or even \emph{previous iterations of an aggregate learning algorithm}.

We summarize our contributions as follows.
\begin{itemize} 
    \item {\bf Reduction to size-constrained $k$-means clustering.} We first present our method assuming access to a prior. We start with linear regression and characterize the dependence of the model estimation error on the bag construction.
    We then show that finding optimal bags reduces to a one-dimensional size-constrained $k$-means clustering problem
    that involves prior information on the \emph{expected} response of samples.
    In Section~\ref{sec:glms}, we then extend our derivations to the family of generalized linear models.
    
    \item{\bf Advantage over random bagging.} In Section~\ref{sec:comparison}, we theoretically demonstrate the improvement of our bagging approach over schemes that construct bags independently of data (including random bagging).

    \item{\bf Iterative prior-boosting algorithm.}
    In Section~\ref{sec:algorithm}, we propose an adaptive algorithm called \texttt{PriorBoost}, which constructs
    a good prior from the aggregate data itself.
    It can be used even in settings where no public prior distribution is available.
    \PriorBoost partitions the training data across multiple stages: it start with random bagging,
    and then \emph{iteratively refines the prior}
    by constructing more consistent bags on the remaining data.

    \item{\bf Differentially private LAR.}
    In Section~\ref{sec:differential_privacy}, we propose a mechanism that adds Laplace noise to aggregate responses to ensure label differential privacy.
    We observe an intriguing tradeoff on the choice of minimum bag size~$k$.
    On the one hand, larger $k$ implies less sensitivity of aggregate responses to individual substitution and hence less noise is needed to ensure privacy.
    On the other hand, smaller~$k$ results in smaller bias of the trained model.
    The optimal choice of $k$ (for a fixed privacy budget $\varepsilon$)
    depends on how these two effects contribute to the model test loss.
    We showcase this tradeoff empirically and discuss how the optimal~$k$ varies with the sample size $n$,
    features dimension~$d$, and bag construction algorithm.

    \item{\bf Experiments.}
    We study \PriorBoost through extensive experiments in Section~\ref{sec:experiments}.
    This includes a comparison with random bagging for linear and logistic regression tasks,
    as well as a careful exploration into label differential privacy with Laplace noise
    for different privacy budgets.

\end{itemize}


\subsection{Other related work}
An active line of work in LAR is centered around the design of new loss functions.
In addition to the the event-level loss in~\eqref{eq:event0},
another popular choice is \emph{bag-level loss} (or aggregate likelihood),
which measures the mismatch between the  aggregate responses $\overline{y}_\ell$ and the \emph{average} model predictions
$\sfrac{1}{|B_\ell|}\sum_{i\in B_\ell} f_{\th}(\bx_i)$ across bags $\ell\in[m]$~\citep{rueping2010svm,yu2014learning}.
\citet{javanmard2024loss} study the statistical properties of both losses and show that for quadratic loss functions $\ell(x,y) = (x-y)^2$, the event-level loss can be seen as a regularized form of the bag-level loss.
They propose a novel interpolating loss that optimally adjusts the strength of the regularization.

It is worth noting that in many large-scale production ML systems, models are often trained online~\cite{anil2022factory,fahrbach2023learning,coleman2023unified},
and event-level loss is more amenable to online optimization.
A separate system can be in charge of bagging and generating aggregate responses without the learner needing to know the bagging structure.
In contrast, bag-level loss minimization requires computing average predictions for each bag, making it more challenging to implement, especially with mini-batch SGD where all samples in a bag must be in the same batch.
We note that the works discussed above mainly consider random bagging.
Closer to our goal,~\citet{chen2023learning} study the problem of bag curation,
but they take a different approach than ours by grouping samples by common features instead of predicted response values.

\section{Warm-up: Linear regression}
\label{sec:linear_regression}

The high-level intuition behind our approach is that useful bagging configurations are ones where aggregate responses are close to their individual responses.
This allows for the estimator to be close to the \emph{empirical risk minimizer} (ERM),
similar to teacher-student knowledge distillation~\citep{hinton2015distilling}.
Our goal is therefore to use available predictions $\ty \approx {\E[y \mid \bx]}$ based on
prior information to construct better bags for the aggregate learner.

To illustrate this idea, we start with a linear regression setup
where response $y_i$ is generated as
\begin{align}\label{eq:linear}
    y_i = \bx_i^\sT \tth +\eps_i\,,\quad \eps_i\sim\normal(0,\sigma^2).
\end{align}
The design matrix is
$\bX = \begin{bmatrix} \bx_1 & \dots & \bx_n \end{bmatrix}^\sT \in \reals^{n\times d}$,
the response vector is $\by = (y_1,\dots, y_n)^\sT$,
and the noise vector is $\beps = (\eps_1,\dots, \eps_n)^\sT$.
We assume $\beps$ is independent of $\bX$, and that $\E[\beps] = \boldsymbol{0}$ and $\E[\beps\beps^\sT] = \sigma^2 \Iden$.
Letting $m$ denote the number of bags,
we encode the assignment of samples to bags with a
matrix $\bS\in\reals^{m\times n}$, where
\begin{align}\label{eq:S}
    \bS_{\ell,i} =
    \begin{cases}
        \frac{1}{\sqrt{|B_\ell|}} & \text{if } i \in B_\ell,\\
        0 & \text{otherwise}.
    \end{cases}
\end{align}

Consider the event-level loss minimizer of~\eqref{eq:event0} with $\mathcal{L}$ being least squares loss,
which we can write as
\begin{align}\label{eq:hth}
    \hth =\argmin_{\th} \frac{1}{n}\twonorm{\bS^\sT\bS \by - \bX\th}^2.
\end{align}

\subsection{Bounding the estimator error}
Our next result characterizes the error of this estimator.
All proofs in this section are deferred to \Cref{app:linear_regression}.

\begin{thm}\label{thm:linear-regression}
If the design matrix $\bX\in\reals^{n\times d}$ has rank $d$,
then for the estimator $\hth$ given by~\Cref{eq:hth}, we have
\begin{align}
\label{eq:Linear_regression}
    \E\Big[ \twonorm{\hth-\tth}^2 ~\Big|~ \bX \Big]
    =&
    \twonorm{(\bX^\sT\bX)^{-1}\bX^\sT (\bS^\sT\bS - \Iden)\bX\tth}^2
    + \sigma^2 \fronorm{(\bX^\sT\bX)^{-1}\bX^\sT\bS^\sT}^2\,.
\end{align}
\end{thm}

An optimal bagging configuration (in the sense of minimizing the estimation error) is one whose matrix~$\bS$ minimizes
\eqref{eq:Linear_regression} among all feasible partitions.
The first term of the right-hand side
is the (conditional) bias of $\hth$ and the second term is its variance.
As we can see, the choice of $\bS$ affects both terms. 

Instead of solving for an optimal $\bS$,
which can be challenging due to its partition structure,
we first develop an upper bound on the error,
and then we minimize this bound over~$\bS$ to give guidance on how to design aggregation bags.

\begin{coro}\label{cor:LinearReg}
The estimation error $\E[ \norm{\hth-\tth}_2^2 \mid \bX ]$
in \Cref{eq:Linear_regression} is upper bounded by
\begin{align*}
    \E\Big[ \twonorm{\hth-\tth}^2 ~\Big|~ \bX \Big] \le 
    \opnorm{(\bX^\sT\bX)^{-1}\bX^\sT}^2
      \parens{ \twonorm{(\bS^\sT\bS-\Iden)\bX\tth}^2 + \sigma^2\min(m,d) } .
\end{align*}
\end{coro}

\subsection{Reducing to size-constrained $k$-means clustering}
Next observe that $\Iden - \bS^\sT \bS$ is a projection matrix given by
\begin{align*}
   (\Iden - \bS^\sT \bS)_{i,j}
   =\begin{cases}
   1 - \frac{1}{|B_\ell|} & \text{if $i,j \in B_\ell$ and $i = j$},\\
   - \frac{1}{|B_\ell|} & \text{if $i, j \in B_\ell$ and $i \ne j$},\\
   0 & \text{otherwise}.
   \end{cases}
\end{align*}
Specifically, $\Iden - \bS^\sT \bS$ is the projection onto the space of vectors that have zero mean within each bag.

Let $\ty_i:= \E[y_i \mid \bx_i] = \bx_i^\sT \th$ be the conditional expected response of sample $\bx_i$
according to the prior model $\th \in \R^d$.
Letting $\tby = (\ty_1,\dots,\ty_n)$, we then have
\begin{align}\label{eq:k_means_objective}
    \twonorm{(\Iden - \bS^\sT \bS) \tby}^2 = \sum_{\ell=1}^m \sum_{i\in B_\ell} (\ty_i - \mu_\ell)^2\,,
\end{align}
where $\mu_{\ell} = \frac{1}{|B_\ell|} \sum_{i\in B_\ell} \ty_i$ is the mean of the entries of $\tby$ in bag $\ell$.
Observe that \eqref{eq:k_means_objective} is the one-dimensional $k$-means objective.

To summarize, let $\mathcal{B}$ denote the set of all partitions of the~$n$ samples.
Minimizing the upper bound in~\Cref{cor:LinearReg} over the set of non-overlapping bags of size at least $k$
amounts to the following optimization problem:
\begin{align*}
    &\min_{(B_1,\dots,B_m) \in \mathcal{B}} \quad \sum_{\ell=1}^m \sum_{i\in B_\ell} (\ty_i - \mu_\ell)^2 + \sigma^2 \min(m,d) \\
    &~\text{subject to} \quad~~ |B_\ell| \ge k \quad \forall \ell \in [m] \nonumber
\end{align*}

This problem exhibits an interesting tradeoff with the number of bags $m$.
The first term in the objective is the \emph{bias of the estimator} $\hth$, which measures the within-bag deviation of~$\tby$.
If we require larger bags (and hence a smaller $m$), this term increases since there will be more heterogeneity within bags. 
Decreasing $m$, however, reduces the second term in the objective, which is the \emph{variance of the estimator} $\hth$.
The reason is that the aggregate responses $\overline{y}_\ell$ are averaged across larger bags and thus have lower variance.
This reduction in the variance of the aggregated responses corresponds to a reduction in the estimator variance.

Focusing on the case where $m\ge d$, we can drop the second term in the objective to get the following one-dimensional
$k$-means clustering problem with minimum size constraints:\footnote{%
More accurately, this is a one-dimensional $m$-means clustering problem with size constraints.
We use $k$ to denote the minimum bag size to agree with the notion of $k$-anonymity.
}
    \begin{align}
    &\min_{(B_1,\dots,B_m) \in \mathcal{B}} \quad \sum_{\ell=1}^m \sum_{i\in B_\ell} (\ty_i - \mu_\ell)^2 \label{eq:1dimKmeans} \\
    &~\text{subject to} \quad~~ |B_\ell| \ge k \quad \forall \ell \in [m] \nonumber
\end{align}

The next result establishes a structural property about optimal solutions to this problem.
\begin{lemma}[Sorting structure]
\label{lem:sorting}
Consider the optimization problem~\eqref{eq:1dimKmeans} and
sort the values $\ty_i$ in non-increasing order as $\ty_{(1)} \ge \dots \ge \ty_{(n)}$.
There exists an optimal solution $\{B^*_\ell : \ell \in [m]\}$ with the following property:
if $\ty_{(i)}$ and $\ty_{(j)}$ are in a bag $B^*_\ell$, then $\ty_{(k)}\in B^*_\ell$ for all $k \in \{i, i+1, \dots, j\}$.
\end{lemma}

\noindent
We discuss the algorithmic consequences of \Cref{lem:sorting}
in more detail in \Cref{sec:algorithm}.

\section{Extension to GLMs}
\label{sec:glms}

We next extend our derivation to the family of \emph{generalized linear models} (GLMs).
In a GLM, the response variables $y_i$ are conditionally independent given $\bx_i$,
and generated from a particular distribution in the exponential family where the log-likelihood function is written as:
\begin{align}\label{eq:GLM}
    \log p(y_i \mid \eta_i,\phi) = \frac{y_i\eta_i -b(\eta_i)}{a_i(\phi)} + c(y_i,\phi)\,,
\end{align}
where $\eta_i$ is the location parameter and $\phi$ is the scale parameter.
The functions $a_i(\cdot)$, $b(\cdot)$, and $c(\cdot,\cdot)$ are known.
It is sometimes assumed that $a_i(\phi)$ 
has the form $a_i(\phi) = \phi/w_i$, where $w_i$ is a known prior weight.
We consider canonical GLMs, in which the location parameter has the form $\eta_i = \bx_i^\sT \tth$ for an unknown model parameter~$\tth$.
GLMs include several well-known statistical models,
including linear regression, logistic regression, and Poisson regression.
  
  Let $\hth$ be the minimizer of the event-level loss in \eqref{eq:event0}
  with $\cL$ the negative log-likelihood. Concretely,
  \begin{align}\label{eq:GLM_estimator}
      \hth
      &= \argmin_{\th} \cL(\th) \notag \\
      &:= \argmin_{\th}\frac{1}{n}\sum_{\ell=1}^m \sum_{i\in B(\ell)}
        \frac{\overline{y}_\ell \bx_i^\sT \th -  b(\bx_i^\sT \th)}{a_i(\phi)}\,,
  \end{align}
 where we drop the term $c(y_i,\phi)$ as it does not depend on $\th$.
 
 By the optimality of $\hth$, we have $\nabla\cL (\hth) = \mathbf{0}$.
 Our goal is to find a bagging configuration that makes $\hth$ close to the ground truth model $\tth$.
 A natural approach towards this goal is to make the gradient of the loss at~$\tth$ small.
 As we show in Lemma~\ref{lem:taylor}, for strongly convex losses, the estimation error $\norm{\hth - \tth}_2$ can be controlled by $\norm{\nabla \cL(\tth)}_2$. 

 Our next result characterizes the norm of the loss gradient at $\tth$,
 connecting it to the bagging matrix~$\bS$.
 Throughout, we use the following convention: For a function $f:\reals\to\reals$,
 when $f$ is applied to a vector, it is applied to each entry of that vector,
 i.e., $f(\bv) = (f(v_1), \dots, f(v_n))$.
 \begin{thm}\label{thm:GLM}
  Consider the GLM family in \eqref{eq:GLM} with canonical link functions
  ($\eta_i = \bx_i^\sT \tth$).
  For negative log-likelihood loss in \eqref{eq:GLM_estimator}, we have
  \begin{align}
  \E\Big[\twonorm{\nabla \cL(\tth)}^2 ~\big|~ \bX \Big]
  &=
  \twonorm{\bX^\sT \bD^{-1} (\bS^\sT\bS  - \Iden) b'(\bX\tth)}^2 \notag \\
  &\hspace{0.4cm}+ \fronorm{\bX^\sT\bD^{-1}\bS^\sT\bS \bD^{1/2}\diag(b''(\bX\tth))^{1/2}}^2\,,\label{eq:nabla-L}
  \end{align}
  where $\bD = \diag(\{a_i(\phi)\})$.
 \end{thm}
 
 We defer all proofs in this section to \Cref{app:glms}.
 Note that $\E[y \mid \bx] = b'(\bx^\sT \tth)$ and $\Var(y \mid \bx) = a(\phi) b''(\bx^\sT \tth)$ are available from the given prior and therefore, in principle, the right-hand side of~\eqref{eq:nabla-L} can be minimized over the choice of bagging matrix $\bS$.
 
 However, similar to the case of linear regression,
 we start by upper bounding \eqref{eq:nabla-L},
 and then we minimize this upper bound over the choice of~$\bS$.
 This provides guidance for how to construct the bags,
 and is easier to compute while being more interpretable.
 
 \begin{coro}
 \label{cor:GLM_estimator_upper_bound}
    Define $\mu_i := \E[y_i \mid \bx_i] = b'(\bx_i^\sT\tth)$ and $v_i := \Var(y_i \mid \bx_i) = a_i(\phi)b''(\bx_i^\sT\tth)$,
    and let their vector forms be $\bmu = (\mu_1,\dotsc, \mu_n)$ and $\bv = (v_1,\dotsc, v_n)$.
    Then,
 \begin{align}
 \label{eq:UB-GLM}
    \E\Big[\twonorm{\nabla \cL(\tth)}^2 ~\big|~ \bX \Big] \le \opnorm{\bX^\sT \bD^{-1}}^2
     \cdot \Big\{ \twonorm{ (\bS^\sT\bS  - \Iden) \bmu}^2
    + \min\Big(\sum_{\ell=1}^m \sum_{i\in B_\ell} \frac{v_i}{|B_\ell|}, d \infnorm{\bv}\Big) \Big\}\,.
 \end{align}
 \end{coro}

In the case of linear regression, we have $v_i = \sigma^2$, 
so the term involving $v_i$ becomes $\sigma^2 \min(m,d)$
like in \Cref{cor:LinearReg},
which only depends on the number of bags.

Further, if $m/d \ge \max(v_i)/\min(v_i)$, the $\min$ term in~\eqref{eq:UB-GLM} is achieved by $d\infnorm{\bv}$,
so this term can be dropped from the objective,
bringing us to the familiar size-constrained clustering problem:
\begin{align}
     &\min_{(B_1,\dots,B_m) \in \mathcal{B}} \quad \sum_{\ell=1}^m \sum_{i\in B_\ell} (\mu_i - \overline{\mu}_\ell)^2 \label{eq:bagOp1}\\
     &~~\text{subject to} \quad~ |B_\ell|\ge k \quad \forall \ell \in [m] \nonumber
 \end{align}
 
We conclude by showing that we can drop the variance term
from the bound in~\eqref{eq:UB-GLM} for logistic and Poisson regression,
i.e., that \eqref{eq:bagOp1} is the correct objective function.

\paragraph{Logistic regression.}
In this case we have $y\in\{0,1\}$, so the log-likelihood becomes:
\[
    \log p(y \mid \eta) = y\eta - \log(1+e^\eta)\,,
\]
which corresponds to $b(\eta) = \log(1+e^\eta)$, $a(\phi) = 1$, and $c(y,\phi) = 0$.
Therefore, $\mu = b'(\eta) = 1/(1+e^{-\eta})$ and $v = e^\eta/(1+e^\eta)^2$.
Then, for any $i,j\in[n]$, we have
\begin{align*}
    \frac{v_i}{v_j} &= e^{\eta_i-\eta_j} \left(\frac{1+e^{\eta_j}}{1+e^{\eta_i}}\right)^2
\le e^{\eta_i-\eta_j} e^{2(\eta_j-\eta_i)_+} \\
  &\le e^{|\eta_i-\eta_j|}\le e^{\twonorm{\bx_i-\bx_j}} e^{\twonorm{\tth}},
\end{align*}
where we used $\eta_i = \bx_i^\sT \tth$.
Therefore, if $\twonorm{\bx_i}\le B$, we have 
$\max(v_i)/\min(v_i)\le \exp(2B\twonorm{\tth})$,
so for $m/d\ge \exp(2B\twonorm{\tth})$,
we can drop the variance term from the objective function.

\paragraph{Poisson regression.}
In this case we have $y \in \mathbb{Z}_{\ge 0}$, so the log-likelihood reads as:
\[
\log p(y \mid \eta) = y\eta - e^{\eta} - \log(y!)\,, 
\]
which corresponds to $b(\eta) = e^\eta$, $c(y,\phi) = -\log(y!)$, and $a(\phi) =1$.
Thus, $\mu = b'(\eta)=e^\eta$ and $v = a(\phi) b''(\eta) = e^\eta$. Then, similar to the previous example,
$\max(v_i)/\min(v_i) \le \exp(2B\twonorm{\tth})$ and so for $m/d\ge \exp(2B\twonorm{\tth})$,
we can drop the variance term from the objective function.

\section{Comparison with random bagging}
\label{sec:comparison}
We now theoretically justify the benefit of our prior-based bagging approach for
aggregate learning compared to random bagging
by proving a separation in the estimator error for linear models.
An analogous but more involved analysis can also be carried out for GLMs.
Before we present our results, we neet to establish some definitions and state our assumptions.

\begin{definition}
A random variable $X$ is $\eta$-subgaussian if 
$\E[\exp(X^2/\eta^2)] \le 2$. A random vector $\bx$ is $\eta$-subgaussian if all of the one-dimensional marginals are $\eta$-subgaussian,
i.e., $\bx^\sT \bv$ is $\eta$-subgaussian for all $\bv$ with $\twonorm{\bv}=1$.
\end{definition}

Some examples of subgaussian random variables
include Gaussian, Bernoulli, and all bounded random variables.

\begin{assumption}\label{ass:SG}
The features vectors $\bx_1,\dotsc,\bx_n\in\reals^d$ are drawn i.i.d from a centered $\kappa$-subgaussian distribution with covariance matrix $\bSigma:=\E[\bx_i\bx_i^\sT]\in\reals^{d\times d}$.
\end{assumption}

\begin{assumption}\label{ass:eig}
We consider an asymptotic regime where the sample size $n$ and the features dimension $d$ both grow to infinity.
We assume that the eignevalues of $\bSigma$ remain bounded and also away from zero in this asymptotic regime,
i.e., $\sigma_{\min}(\bSigma)\ge C_{\min}>0$ and $\opnorm{\bSigma} \le C_{\max} <\infty$ for some constants $C_{\min}$ and $C_{\max}$.
\end{assumption}
Our first theorem upper bounds the estimator error when the bags are formed using the ground truth model $\tth$.

\begin{thm}\label{thm:UB-risk}
Consider the linear model~\eqref{eq:linear} under Assumptions~\ref{ass:SG} and \ref{ass:eig}. Suppose that the dimension $d$ and the sample size $n$ grow to infinity and $n=\Omega(d)$. For the bagging matrix $\bS$ constructed by solving problem~\eqref{eq:1dimKmeans}, the following holds true with probability at least $1-1/n - 2e^{-cd}$,
\[
\E\Big[\twonorm{\hth-\tth}^2 ~\Big|~ \bX\Big] \le C\left(\frac{k\log(n) \twonorm{\tth}^2 + \sigma^2 d}{n\sigma_{\min}(\bSigma)}\right) \,,
\]
for some constants $c, C>0$ that depend only on the subgaussian norm $\kappa$.
\end{thm}

Out next result lower bounds the estimator error when the bags are chosen \emph{independently of the data}.
This applies to random bags as a special case.

\begin{thm}\label{thm:LB-risk}
Consider the linear model~\eqref{eq:linear} under Assumptions~\ref{ass:SG} and \ref{ass:eig}. Suppose the dimension $d$ and the sample size $n$ grow to infinity and $n = \Omega(d^2 \log d)$.  If the bags are constructed independent of data and each of size $k$, the following holds true with probability at least $1-2e^{-c_1d} - 2d^{-c}$,
\begin{align*}
    &\E\Big[\twonorm{\hth-\tth}^2 ~\Big|~ \bX\Big] \ge
    \Bigg[ \left(1-\frac{1}{k} - \frac{Cd\sqrt{\log d}}{\sigma_{\min}(\bSigma)\sqrt{n}}\right)^2 \twonorm{\tth}^2
    + \frac{\sigma^2}{kn}  \frac{\trace(\bSigma) - \sqrt{d\log d}}{(\opnorm{\bSigma} + c_0\sqrt{\frac{d}{n}})^2} ~\Bigg] \,,
\end{align*}
where $c,c_0,c_1,C>0$ are constants that only depend on~$\kappa$, the subgaussian norm of the features vectors.
\end{thm}

\begin{remark}
Theorems~\ref{thm:UB-risk} and \ref{thm:LB-risk} quantify the improvement we get in model risk when using the bag construction from constrained $k$-means instead of random bags. Note that in the asymptotic regime where $n,d\to\infty$ with $n = \Omega(d^2\log d)$, the model risk under \Cref{thm:UB-risk} converges to zero, while the risk under \Cref{thm:LB-risk} is lower bounded by $(1-\frac{1}{k})^2 \twonorm{\tth}^2$.
We note that $C_{\min} d\le \trace(\bSigma)\le c_{\max} d$. In other words, the bias of the estimated model remains non-vanishing under random bags, whereas it vanishes asymptotically when the bags are constructed via size-constrained $k$-means. 
\end{remark}

\begin{remark}
Theorem~\ref{thm:UB-risk} considers bagging configurations based on $k$-means with a minimum group size constraint in~\eqref{eq:1dimKmeans}. It assumes access to an oracle model that gives the correct ordering of (unobserved) responses $y_i$.
However, as stated in our methodology, we use a prior model to compute the conditional expected responses $\ty_i$,
and because of this there may be a mismatch between the ordering of $y_i$'s and $\ty_i's$.
We denote by $\bS$ and $\widetilde{\bS}$ the corresponding bagging matrices.
Our next theorem shows how the estimator error inflates with respect to the mismatch quantity
$\norm{\bS\bS^\sT - \widetilde{\bS}\widetilde{\bS}^\sT}_{\rm{op}}$.     

\end{remark}

\begin{thm}\label{thm:UB-risk-approximate}
Consider the linear model~\eqref{eq:linear} under Assumptions~\ref{ass:SG} and \ref{ass:eig}. Suppose that the dimension $d$ and the sample size $n$ grow to infinity, and $n=\Omega(d)$. Let $\widetilde{\bS}$ be the bagging configuration based on problem~\eqref{eq:1dimKmeans}
using the predicted responses $\ty_i$ by a prior model.
Similarly, let $\bS$ be the corresponding bagging configuration by an oracle model who has access to individual responses $y_i$.
If we have a mismatch $\norm{\bS\bS^\sT - \widetilde{\bS}\widetilde{\bS}^\sT}_{\rm{op}} \le \eps$, then the following holds true with probability at least $1-1/n-2e^{-cd}$,
\begin{align*}
    &\E\Big[\twonorm{\hth-\tth}^2 ~\Big|~ \bX\Big] \le C\left(\frac{k\log(n) \twonorm{\tth}^2 + \sigma^2 d}{n\sigma_{\min}(\bSigma)}\right)
      + \frac{\sigma_{\max}(\bSigma)- C'\sqrt{d/n}}{\sigma_{\min}(\bSigma)+ C'\sqrt{d/n}}\twonorm{\tth}^2 \eps^2 ,
\end{align*}
for some constants $c, C, C'>0$ that depend only on the subgaussian norm $\kappa$.
\end{thm}

\section{Algorithm}
\label{sec:algorithm}

We now present the \PriorBoost algorithm.
The high-level idea is to partition the data $\bX$ into $T$ parts,
and use each slice $\bX^{(t)}$ together with last round's model $\hth^{(t-1)}$
to form better bags $\bS^{(t)}$, and hence learn a stronger event-level model~$\hth^{(t)}$ at each step.
This is an iterative and adaptive procedure.
However, since we get one aggregate response per sample (non-overlapping bags), taking more steps means using less data per step.
We compare \PriorBoost to the random bagging algorithm in \Cref{sec:experiments}
that uses all available data in a one non-adaptive round.

Concretely, the first step of \PriorBoost
uses random bagging to learn $\hth^{(1)}$
from the aggregate responses of the first slice $\bX^{(1)}$.
In each subsequent step, we use $\hth^{(t-1)}$
to predict the individual responses $\tby^{(t)}$ for this round of data $\bX^{(t)}$.
Based on these predictions, we form aggregation bags by solving the one-dimensional
size-constrained $k$-means clustering problem in~\eqref{eq:1dimKmeans}.
Recall that our goal is for bags to be homogoneous with respect to the true responses,
which the learner never sees.
The learner then gets the aggregate response of each bag,
learns a better model $\th^{(t)}$, and repeats the process.
We give pseudocode for \PriorBoost in \Cref{alg:priorboost}
and summarize its core clustering subroutine below.

\begin{algorithm}[H]
   \caption{\PriorBoost}
   \label{alg:priorboost}

   \textbf{Input:}
   data $\bX$, model $\cL(\cdot,f_{\th}(\cdot))$, number of steps $T$
   
    \begin{algorithmic}[1]
    \STATE Split $\bX$ into $T$ equal-sized parts $\bX^{(1)}, \dots, \bX^{(T)}$ \label{line:split_data}
    \STATE Get aggregate responses $\overline{\by}^{(1)}$ for $(\bX^{(1)}, \bS^{(\textnormal{random})})$
    \STATE Update $\hth^{(1)} \gets \argmin_{\th} \mathcal{L} (\overline{\by}^{(1)}, f_{\th}(\bX^{(1)}))$
    \FOR{$t=2$ to $T$}
        \STATE Predict $\tby^{(t)} \gets f_{\hth^{(t-1)}}(\bX^{(t)})$
        \STATE Sort samples by $\ty^{(t)}_i$ and solve~\eqref{eq:1dimKmeans} to get bags $\bS^{(t)}$ using \Cref{lem:fast_k_means_solve}
        \STATE Get aggregate responses $\overline{\by}^{(t)}$ for $(\bX^{(t)}, \bS^{(t)})$
        \STATE Update $\hth^{(t)} \gets \argmin_{\th} \mathcal{L} (\overline{\by}^{(t)}, f_{\th}(\bX^{(t)}))$ \label{line:prior_boost_update}
    \ENDFOR

   \STATE \textbf{return} $\hth^{(T)}$
\end{algorithmic}
\end{algorithm}

\begin{lemma}
\label{lem:fast_k_means_solve}
The clustering problem in \eqref{eq:1dimKmeans}
with bags of minimum size $k$
can be solved in time $O(nk + n \log n)$.
\end{lemma}

This subroutine exploits the sorted structure of an optimal partition (\Cref{lem:sorting})
and uses dynamic programming with a constant-time update for the sum of squared distance term
for the last cluster in the recurrence~\citep{wang2011ckmeans}.
We describe this algorithm in more detail and give a proof of the lemma in \Cref{app:algorithm}.

\begin{remark}
If we have a weak model for predicting event-level responses
(e.g., using prior $\hth^{(0)}$ or transfer learning),
we can use its predictions for $\ty_i$ to sort $\bX^{(1)}$
and apply~\Cref{lem:fast_k_means_solve} in step $t=1$.
This warm starts \PriorBoost compared to random bagging
$\bS^{(\textnormal{random})}$
and allows the algorithm to use fewer adaptive rounds.
\end{remark}

\section{Differential privacy for aggregate responses}
\label{sec:differential_privacy}

As previously explained, aggregate learning offers a degree of privacy protection by obscuring individual responses and only disclosing aggregated responses for each bag. If the bags do not overlap and each bag has a minimum size $k$, substituting individual responses with the aggregated ones ensures $k$-anonymity, a privacy concept asserting that any given response is indistinguishable from at least $k-1$ other responses.

Another widely used notion of privacy that formalizes the privacy protection of responses/labels is \emph{label differential privacy} (label DP), introduced by \citet{chaudhuri2011sample}. In simple terms, a mechanism, or data processing algorithm, is deemed label DP if its output distribution remains largely unchanged if a single response/label is altered in the input dataset.
The concept of label differential privacy is derived from (full) differential privacy~\citep{dwork2006our,dwork2006calibrating}, focusing specifically on preserving the privacy of responses rather than all features. It is important to note that differential privacy provides a guarantee for data processing algorithms, whereas $k$-anonymity is a property of datasets. We recall the formal definition of label DP from \citep{chaudhuri2011sample}. %

\begin{definition}[Label differential privacy]
Consider a randomized mechanism $\mathcal{M}: \mathcal{D} \to \mathcal{O}$ that takes as input dataset $D$
and outputs into space $\mathcal{O}$.
A mechanism $\mathcal{M}$ is called \emph{$\varepsilon$-label DP} if for any two datasets $(D, D')$ that differ in the label
of a single example and any subset $O\subseteq\mathcal{O}$, we have
\[
    \Pr(\mathcal{M}(D)\in O)
    \le
    e^{\varepsilon} \Pr(\mathcal{M}(D')\in O)\,,
\]
where $\varepsilon$ is the privacy budget.
\end{definition}

It is easy to see that learning from aggregate responses, in the form described so far, is not label DP.
However, we can use the Laplace mechanism on top of aggregation to ensure label DP.
We empirically study the optimal size of the bags,
in terms of minimizing model estimation error, for a given privacy budget
in \Cref{sec:experiments_dp}.

In the Laplace mechanism, the magnitude of the noise being added depends
on the privacy guarantee $\eps$ and the sensitivity of the output to each single change in the dataset. Suppose that the responses/labels are bounded $|y_i|\le B$ by some value $B$ that is independent of data
(and hence can be used without sacrificing any data privacy).
The sensitivity of an aggregate response, for a bag of size $k$, is then given by $B/k$.
Therefore, to ensure $\eps$-label DP, we add independent draws
$Z_\ell\sim {\rm Laplace}(0, B/\eps k)$ to the aggregate responses~$\overline{y}_\ell$,
for each  $\ell\in[m]$.
By \citet[Proposition 1]{dwork2006calibrating} these noisy aggregated responses are $\eps$-DP,
and by closure of DP under post-processing~\citep[Proposition 2.1]{dwork2014algorithmic} any learning algorithm that only uses the noisy aggregate responses is $\eps$-DP.

\section{Experiments}
\label{sec:experiments}

We empirically study linear regression, logistic regression, and label DP
in the aggregate learning framework.
For these tasks, we compare three algorithms:
\begin{itemize}
    \item \texttt{PriorBoost}: Pseudocode presented in \Cref{alg:priorboost}.
    \item \texttt{OneShot}:
        Random bagging on all of the training data.
        This is equivalent to \Cref{alg:priorboost} with $T=1$
        (i.e., a non-adaptive version).
    \item \texttt{PBPrefix}:
        Variant of \Cref{alg:priorboost} where at each step $t$, the model trains on all data seen so far.
        Specifically, the data used to learn $\hth^{(t)}$
        in Line~\ref{line:prior_boost_update} is $\bigcup_{i=1}^t \{(\bX^{(i)}, \overline{\by}^{(i)})\}$.
\end{itemize}

\noindent
Our experiments use NumPy~\citep{harris2020array} and scikit-learn's \texttt{LogisticRegression} classifier~\citep{scikit-learn}.

\subsection{Linear regression}

We start by generating a dataset
$(\bX, \by)$ with $\bX \in \reals^{n \times d}$
as follows.
First, sample a ground truth model $\tth \sim \mathcal{N}_{d}(\mathbf{0}, \Iden)$.
Next, generate a design matrix $\mat{X}$ of $n$ i.i.d.\ feature vectors $\bx_{i} \sim \mathcal{N}_{d}(\mathbf{0}, \Iden)$
and get their responses $\by = \bX \tth + \beps$,
where each $\eps_i \sim \mathcal{N}(0, \sigma^2)$ is i.i.d.\ Gaussian noise with $\sigma = 0.1$.

To study the convergence of \texttt{PriorBoost} and \texttt{PBPrefix},
we set $T = 256$.
Then we set $n = T \cdot 4096 = 2^{20}$ and $d = 8$
so that both algorithms get $4096$ new samples per step.
We generate an independent test set of $n$ samples from the same model
and plot the test mean squared error (MSE)
at each step of the algorithm (using the entire test set) in \Cref{fig:linear_regression}.
\texttt{OneShot}, as described above, creates random bags of size $k$ across all of the training data,
gets the mean response of each bag,
and fits a linear regression model with least squares loss.
We run each algorithm for bags of size $k \in \{1,2,4,8,16,32,64\}$.

\begin{figure}[t]
\centering
\hspace{-0.5cm}
\includegraphics[width=0.48\textwidth]{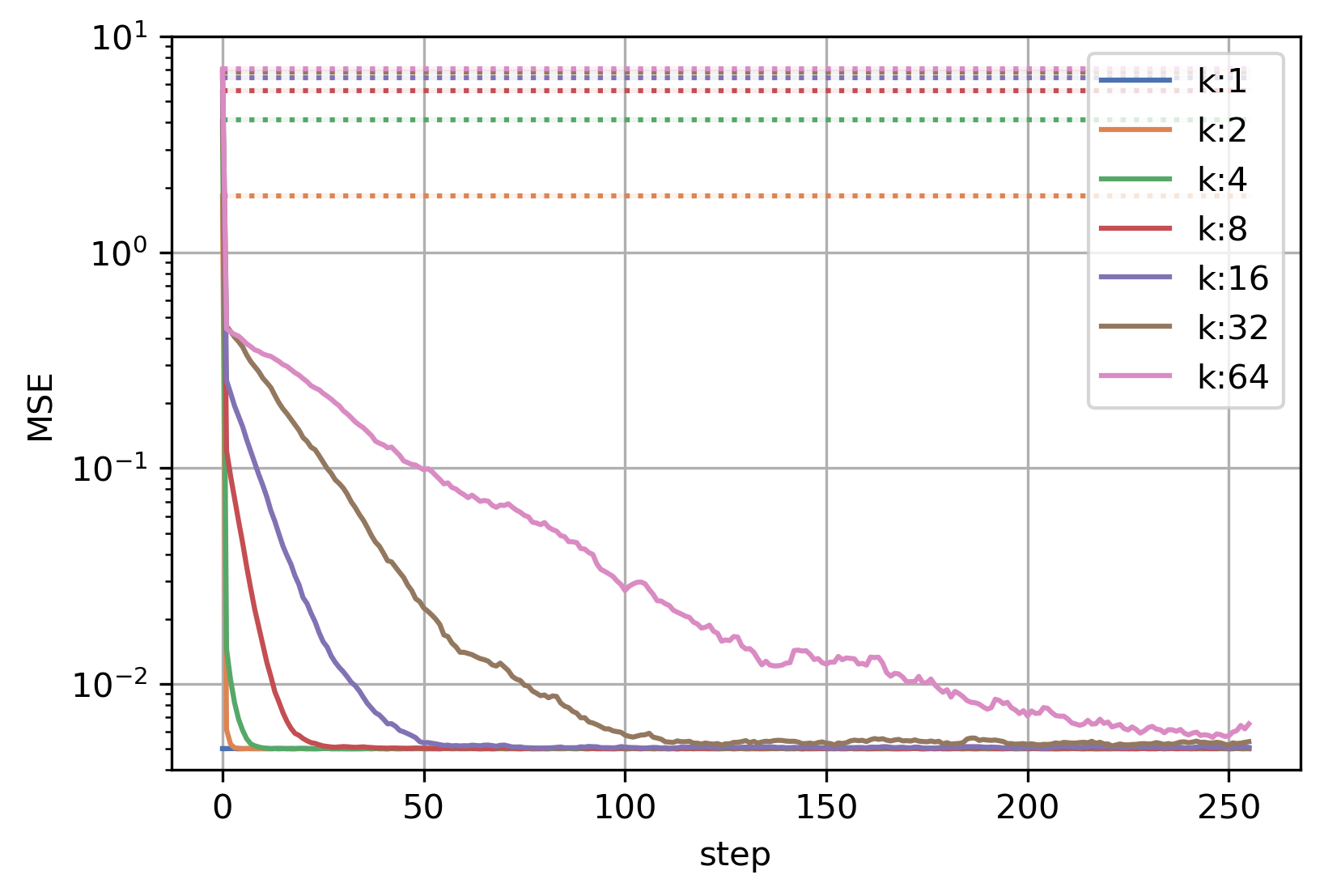}
\includegraphics[width=0.48\textwidth]{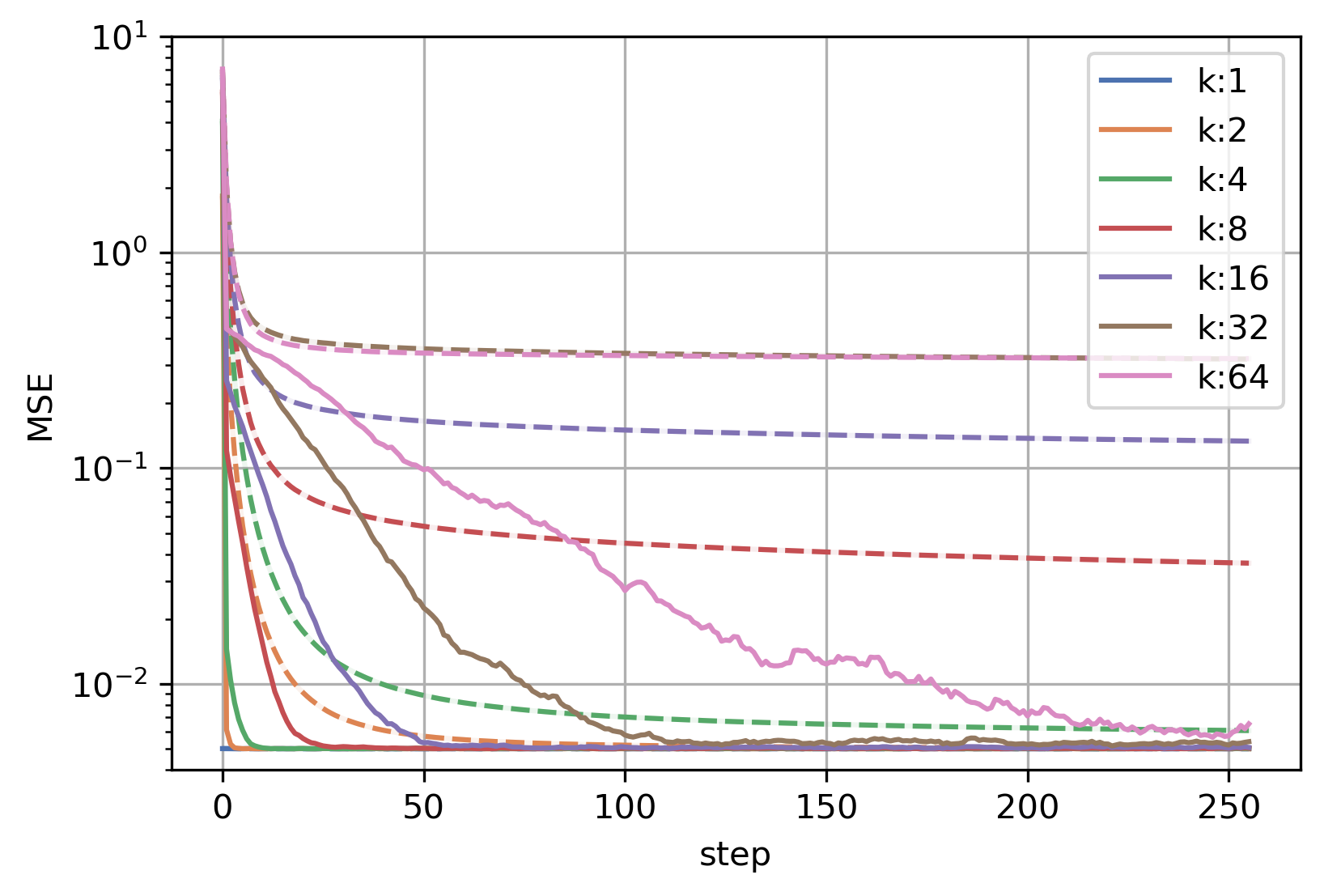}
\caption{Linear regression.
Compares
\texttt{PriorBoost} (solid) with
\texttt{OneShot} (left, dotted) and
\texttt{PBPrefix} (right, dashed)
by plotting
test MSE at each step $t$
for different bag sizes $k$.
}
\label{fig:linear_regression}
\end{figure}

In \Cref{fig:linear_regression}, 
\PriorBoost converges to optimal model quality (i.e., the loss when $k=1$)
for all bag sizes~$k$.
This is in stark contrast to non-adaptive \texttt{OneShot} (i.e., random bagging),
whose test loss gets worse as~$k$ increases.
In the right subplot,
\texttt{PBPrefix} converges much slower than
\texttt{PriorBoost}---and to suboptimal solutions.
This is because the aggregate responses obtained in early steps of the algorithms
are noisy, as the prior $\hth^{(t)}$ is weaker.
Noisy responses are helpful for constructing better bags
in the next iteration (and hence allowing us to learn a stronger prior),
but they can be actively unhelpful if they remain in the training set for too long
(e.g., the data that \texttt{PBPrefix} trains on in later steps).
\texttt{PriorBoost}, however,
only trains on the last slice of aggregate data $(\bX^{(t)},\overline{\by}^{(t)})$,
and therefore ``forgets'' early/noisy mean responses,
leading to better final model quality while also using fewer samples per step.

\subsection{Logistic regression}
\label{sec:experiments_logistic_regression}

For our first logistic regression experiment,
we use the same ground truth model weights,
design matrix,
and Gaussian noise $(\tth, \bX, \beps)$
as in linear regression,
but now we create binary labels by sending them through a sigmoid function and rounding:
$
    y_i = \texttt{round}(\sigma(\bx_i^\sT \tth + \eps_i)) \in \{0, 1\}.
$
After each aggregation step $t$, the oracle rounds the mean
response of each bag $\texttt{round}(\overline{y}_\ell) \in \{0, 1\}$ to get back to a binary label,
which is an additional source of noise.\footnote{%
We randomly round $\overline{y}_\ell = \sfrac{1}{2}$ to $0$ or $1$ in a consistent way
to avoid biasing the distribution of binary aggregate labels.
}
All three algorithms fit logistic regression models with binary cross-entropy loss
and L2 regularization penalty $\frac{\lambda}{2} \norm{\th}_{2}^2$ for $\lambda = 10$.

\begin{figure}[t]
\centering
\hspace{-0.5cm}
\includegraphics[width=0.48\textwidth]{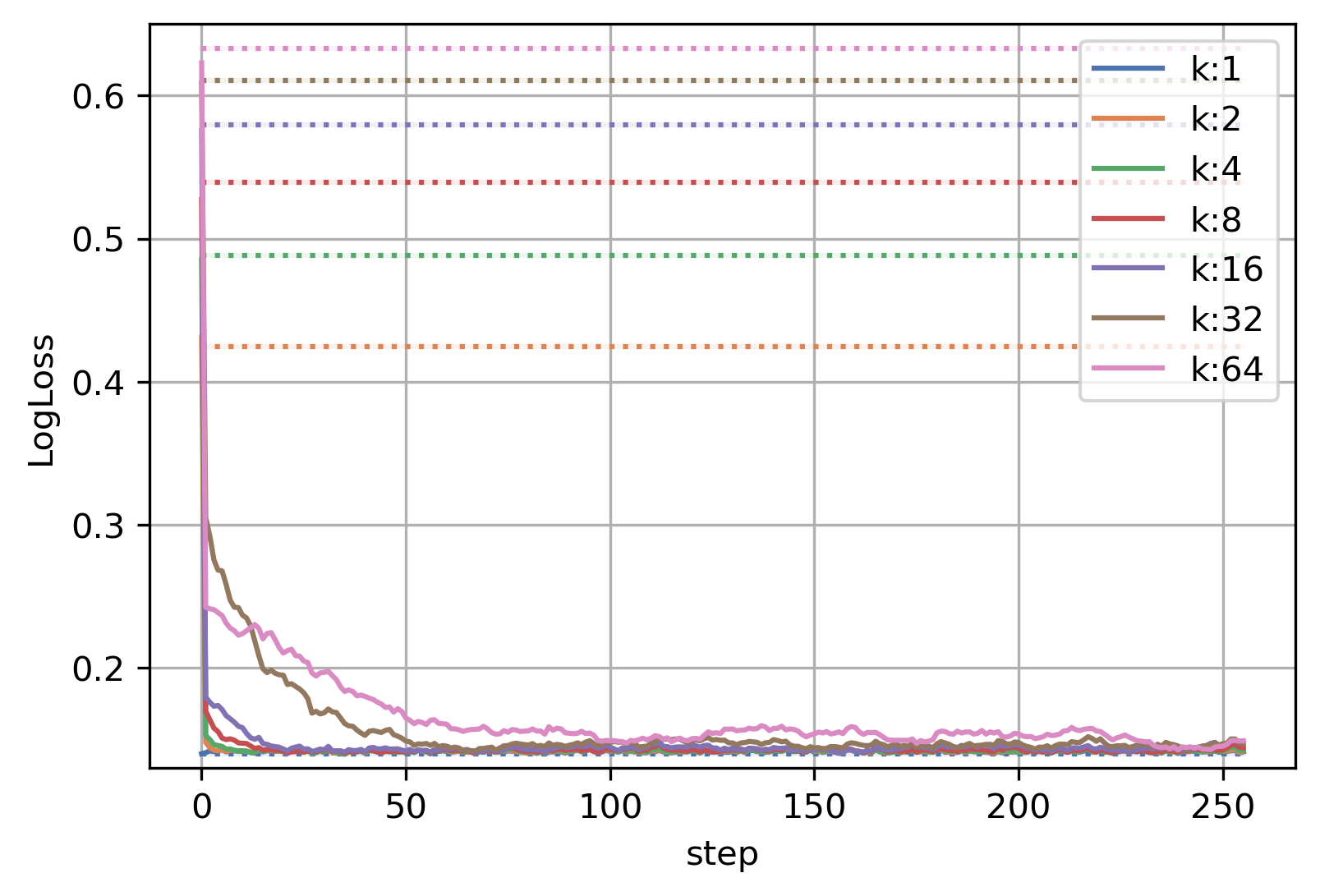}
\includegraphics[width=0.48\textwidth]{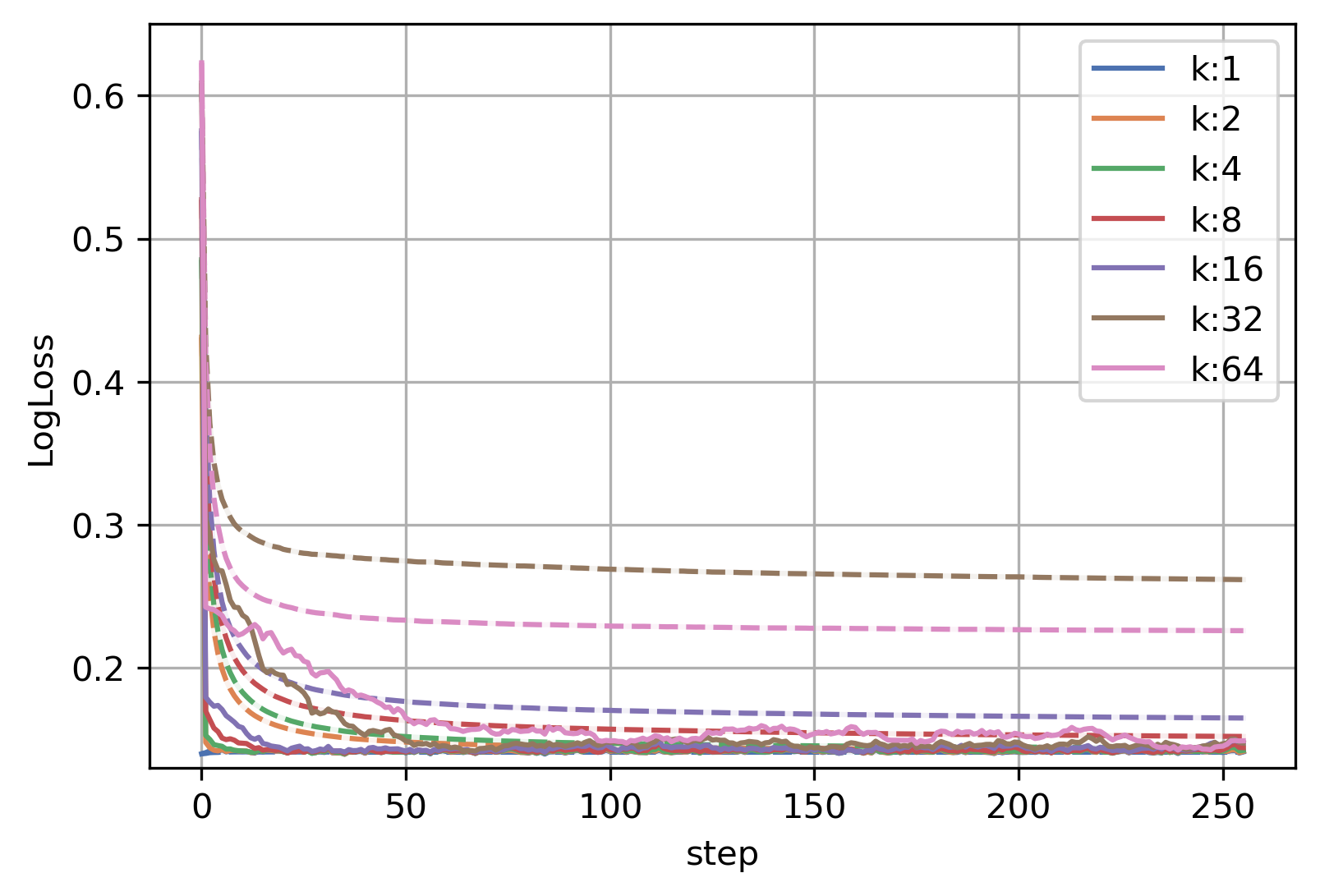}
\caption{Logistic regression.
Compare
\texttt{PriorBoost} (solid) with
\texttt{OneShot} (left, dotted) and
\texttt{PBPrefix} (right, dashed)
by plotting
test log loss at each step $t$
for different bag sizes $k$.
}
\label{fig:logistic_regression}
\end{figure}

Similar to the linear regression experiment,
\Cref{fig:logistic_regression} shows that \PriorBoost
converges to optimality for all bag sizes.
In contrast, \texttt{OneShot} steadily degrades as~$k$ increases.
We also see that by training on all aggregate responses available at step $t$
to learn $\hth^{(t)}$,
\texttt{PBPrefix}
converges slower and to suboptimal solutions for $k \ge 16$.

\subsection{Differential privacy}
\label{sec:experiments_dp}

We now modify \PriorBoost by adding Laplace noise to the aggregate responses
to make the algorithm $\varepsilon$-label DP as described in \Cref{sec:differential_privacy}.
The key observation is that for binary labels and bags of size at least $k$,
we can reduce the scale of the Laplace noise by a factor of $k$.
We use the same experimental setup as in \Cref{sec:experiments_logistic_regression},
geometrically sweep over privacy budgets $0.01 \le \varepsilon \le 100$,
and compare the test loss of
\PriorBoost to an $\varepsilon$-label DP version of random bagging.
Error bars are computed over 10 realizations.

\begin{figure}[H]
\centering
\includegraphics[width=0.48\textwidth]{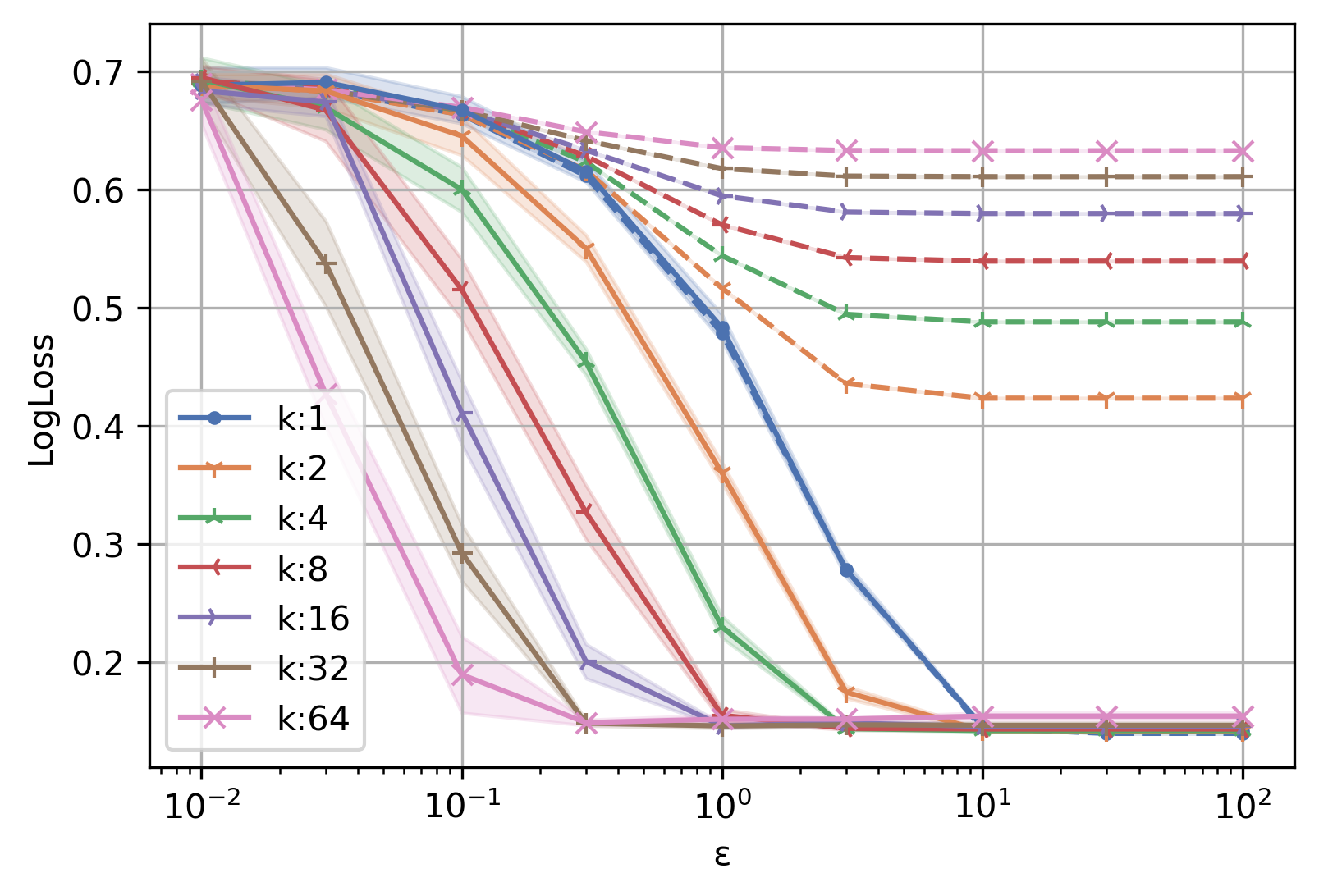}
\caption{$\varepsilon$-label differentially private logistic regression.
Compares final \PriorBoost (solid)
test log loss to \texttt{OneShot} (dashed) for different bag sizes $k$.}
\label{fig:label_dp_epsilon_sweep}
\end{figure}

As we increase the privacy loss $\varepsilon$, the test loss decreases for each value of $k$.
This is expected because the privacy constraint becomes more relaxed, allowing for better model quality.
Note that $\varepsilon = \infty$ corresponds to not using any differential privacy.
For each value of $\varepsilon$,
the utility of \PriorBoost \emph{improves}~as the bag size $k$ increases,
whereas the utility of \texttt{OneShot} degrades as we increase $k$.
This may initially seem surprising,
but recall that \PriorBoost actively reduces the bias of the estimated model by forming homogeneous bags with respect to labels.
As we increase $k$, \PriorBoost (1) effectively maintains a slow growth rate for the bias,
(2) can afford to reduce the scale of its Laplace noise by a factor of $k$,
and (3) gets low-variance mean labels since they are averaged over larger bags.
Therefore, all in all, \PriorBoost favors larger $k$ in this setup.
In particular, it approaches its non-private loss at a faster rate in $\varepsilon$, for larger $k$. 
For example with $k=64$, it already nearly achieves the non-private loss for $\varepsilon \ge 0.3$.  
For \texttt{OneShot} though, larger~$k$ significantly increases the bias of the estimated model, which outweighs the variance reduction and results in a larger loss. Also note that for $k=1$ (singleton bags), \PriorBoost and \texttt{OneShot} match.
In summary, this experiment shows we can achieve more utility for a fixed $\varepsilon$
by using \PriorBoost to learn large curated bags
whose mean labels require less random noise.

\begin{figure}[H]
\centering
\includegraphics[width=0.48\textwidth]{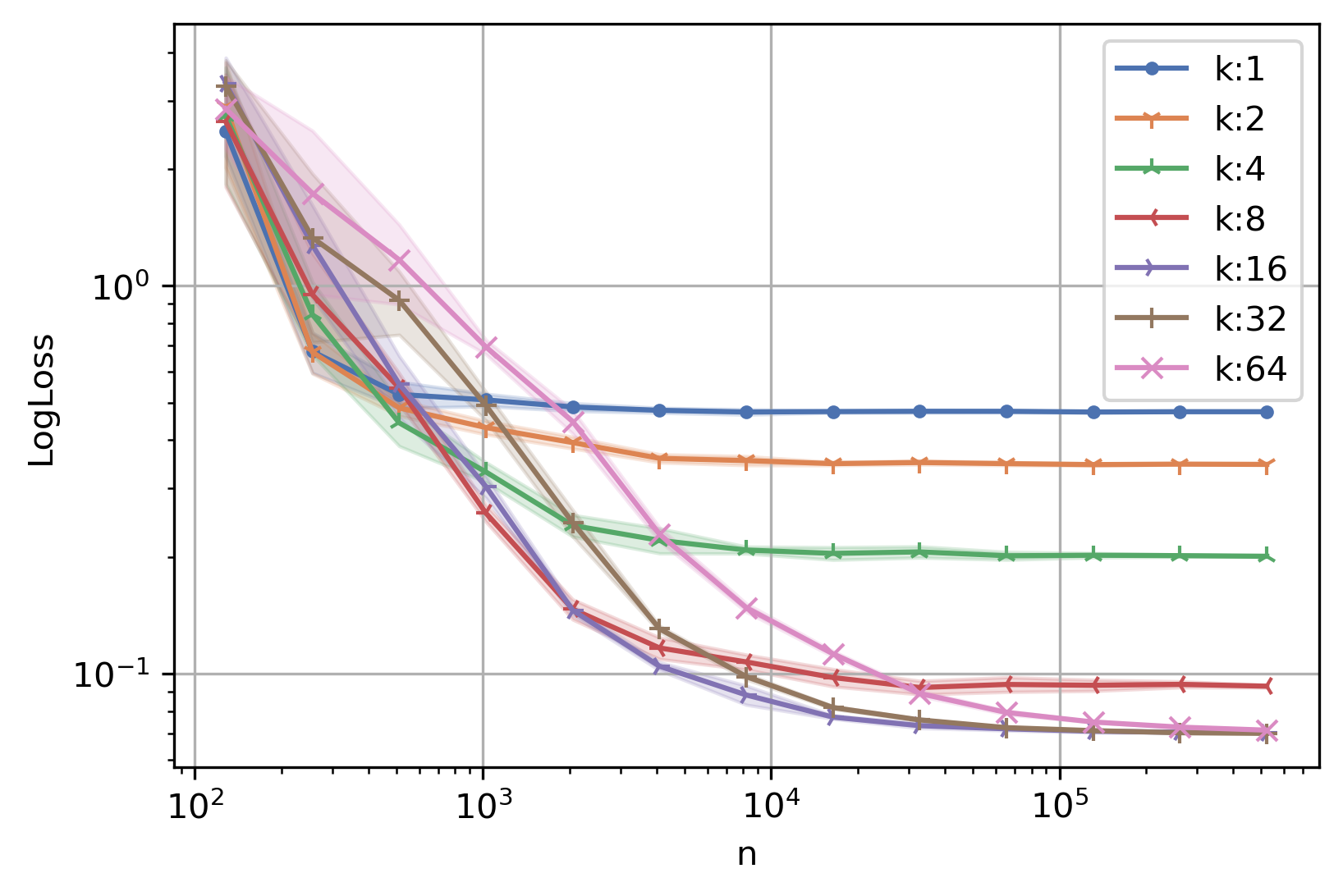}
\caption{Optimal bag sizes $k$ for $\varepsilon$-label DP \texttt{PriorBoost} for logistic regression.
Test loss for $\varepsilon = 1$ as the number of samples $n$ increases.
}
\label{fig:optimal_bag_sizes}
\end{figure}

\paragraph{Optimal bag sizes.}
The plots in Figure~\ref{fig:label_dp_epsilon_sweep} are for fixed~$n$ and $d$ with $n\gg d$. To better understand the effect of bag size on the bias-variance tradeoff, we next run the 
same logistic regression experiment for $d=64$, $T=128$, $\varepsilon=1$,
while varying the total number of samples $n$. 
An intriguing observation from Figure~\ref{fig:optimal_bag_sizes} is that the optimal bag size (i.e., the one minimizing the loss) grows with $n$. The crossover points for optimal $k$ also become farther apart as $n$ grows. For example, $k=4$ is optimal for a smaller range of $n$ compared to $k=16$. As discussed earlier, a larger $k$ yields larger bias while reducing the variance.
However, a virtue of \PriorBoost is that both the bias and variance decrease in the sample size $n$.
The loss plots in \Cref{fig:optimal_bag_sizes} suggest that the decay rate (in $n$) of the bias is faster
than the decay rate of the variance, and so as $n$ grows, the optimal bag size for \PriorBoost becomes larger.

\section*{Conclusion}
This work proposes a novel method for using available prior information for expected responses of samples to construct bags for aggregate learning.
We devise the multi-stage algorithm \PriorBoost to obtain good priors from the aggregate data itself if no public prior is available.
We also propose a differentially private version, as well as intriguing observations about optimal bag sizes.
Our analysis provably shows the advantage of our approach over random bagging, which we back up with strong numerical experiments.

\section*{Acknowledgement}
We would like to thank
Lorne Applebaum,
Ashwinkumar Badanidiyuru,
Lin Chen, Alessandro Epasto, Thomas Fu, Xiaoen Ju, Nick Ondo, Fariborz Salehi, and Dinah Shender
for helpful discussions related to this work.
Adel Javanmard is supported in part by the NSF CAREER Award DMS-1844481 and the NSF Award DMS-2311024.

\bibliographystyle{abbrvnat}
\bibliography{references}

\begin{thebibliography}{30}
\providecommand{\natexlab}[1]{#1}
\providecommand{\url}[1]{\texttt{#1}}
\expandafter\ifx\csname urlstyle\endcsname\relax
  \providecommand{\doi}[1]{doi: #1}\else
  \providecommand{\doi}{doi: \begingroup \urlstyle{rm}\Url}\fi

\bibitem[Anil et~al.(2022)Anil, Gadanho, Huang, Jacob, Li, Lin, Phillips, Pop, Regan, Shamir, et~al.]{anil2022factory}
R.~Anil, S.~Gadanho, D.~Huang, N.~Jacob, Z.~Li, D.~Lin, T.~Phillips, C.~Pop, K.~Regan, G.~I. Shamir, et~al.
\newblock On the factory floor: {ML} engineering for industrial-scale ads recommendation models.
\newblock In \emph{Proceedings of the 5th Workshop on Online Recommender Systems and User Modeling co-located with the 16th {ACM} Conference on Recommender Systems}. CEUR-WS, 2022.

\bibitem[Brahmbhatt et~al.(2023)Brahmbhatt, Pokala, Saket, and Raghuveer]{brahmbhatt2023llp}
A.~Brahmbhatt, M.~Pokala, R.~Saket, and A.~Raghuveer.
\newblock {LLP-Bench}: {A} large scale tabular benchmark for learning from label proportions.
\newblock \emph{arXiv preprint arXiv:2310.10096}, 2023.

\bibitem[Busa-Fekete et~al.(2023)Busa-Fekete, Choi, Dick, Gentile, and Medina]{busa2023easy}
R.~I. Busa-Fekete, H.~Choi, T.~Dick, C.~Gentile, and A.~M. Medina.
\newblock Easy learning from label proportions.
\newblock \emph{arXiv preprint arXiv:2302.03115}, 2023.

\bibitem[Chaudhuri and Hsu(2011)]{chaudhuri2011sample}
K.~Chaudhuri and D.~Hsu.
\newblock Sample complexity bounds for differentially private learning.
\newblock In \emph{Proceedings of the 24th Annual Conference on Learning Theory}, pages 155--186. JMLR, 2011.

\bibitem[Chen et~al.(2023)Chen, Fu, Karbasi, and Mirrokni]{chen2023learning}
L.~Chen, G.~Fu, A.~Karbasi, and V.~Mirrokni.
\newblock Learning from aggregated data: {C}urated bags versus random bags.
\newblock \emph{arXiv preprint arXiv:2305.09557}, 2023.

\bibitem[Coleman et~al.(2023)Coleman, Kang, Fahrbach, Wang, Hong, Chi, and Cheng]{coleman2023unified}
B.~Coleman, W.-C. Kang, M.~Fahrbach, R.~Wang, L.~Hong, E.~H. Chi, and D.~Z. Cheng.
\newblock Unified {E}mbedding: {B}attle-tested feature representations for web-scale {ML} systems.
\newblock \emph{arXiv preprint arXiv:2305.12102}, 2023.

\bibitem[de~Freitas and K{\"{u}}ck(2005)]{DBLP:conf/uai/FreitasK05}
N.~de~Freitas and H.~K{\"{u}}ck.
\newblock Learning about individuals from group statistics.
\newblock In \emph{Proceedings of the 21st Conference in Uncertainty in Artificial Intelligence}, pages 332--339. {AUAI} Press, 2005.

\bibitem[Dwork et~al.(2006{\natexlab{a}})Dwork, Kenthapadi, McSherry, Mironov, and Naor]{dwork2006our}
C.~Dwork, K.~Kenthapadi, F.~McSherry, I.~Mironov, and M.~Naor.
\newblock Our data, ourselves: {P}rivacy via distributed noise generation.
\newblock In \emph{Proceedings of the 25th Annual International Conference on the Theory and Applications of Cryptographic Techniques}, pages 486--503. Springer, 2006{\natexlab{a}}.

\bibitem[Dwork et~al.(2006{\natexlab{b}})Dwork, McSherry, Nissim, and Smith]{dwork2006calibrating}
C.~Dwork, F.~McSherry, K.~Nissim, and A.~Smith.
\newblock Calibrating noise to sensitivity in private data analysis.
\newblock In \emph{Proceedings of the Third Theory of Cryptography Conference}, pages 265--284. Springer, 2006{\natexlab{b}}.

\bibitem[Dwork et~al.(2014)Dwork, Roth, et~al.]{dwork2014algorithmic}
C.~Dwork, A.~Roth, et~al.
\newblock The algorithmic foundations of differential privacy.
\newblock \emph{Foundations and Trends in Theoretical Computer Science}, 9\penalty0 (3--4):\penalty0 211--407, 2014.

\bibitem[Fahrbach et~al.(2023)Fahrbach, Javanmard, Mirrokni, and Worah]{fahrbach2023learning}
M.~Fahrbach, A.~Javanmard, V.~Mirrokni, and P.~Worah.
\newblock Learning rate schedules in the presence of distribution shift.
\newblock In \emph{Proceedings of the 40th International Conference on Machine Learning}, pages 9523--9546. PMLR, 2023.

\bibitem[Geradin et~al.(2020)Geradin, Katsifis, and Karanikioti]{geradin2020google}
D.~Geradin, D.~Katsifis, and T.~Karanikioti.
\newblock Google as a de facto privacy regulator: Analyzing chrome’s removal of third-party cookies from an antitrust perspective.
\newblock 2020.

\bibitem[Harris et~al.(2020)Harris, Millman, van~der Walt, Gommers, Virtanen, Cournapeau, Wieser, Taylor, Berg, Smith, Kern, Picus, Hoyer, van Kerkwijk, Brett, Haldane, del R{\'{i}}o, Wiebe, Peterson, G{\'{e}}rard-Marchant, Sheppard, Reddy, Weckesser, Abbasi, Gohlke, and Oliphant]{harris2020array}
C.~R. Harris, K.~J. Millman, S.~J. van~der Walt, R.~Gommers, P.~Virtanen, D.~Cournapeau, E.~Wieser, J.~Taylor, S.~Berg, N.~J. Smith, R.~Kern, M.~Picus, S.~Hoyer, M.~H. van Kerkwijk, M.~Brett, A.~Haldane, J.~F. del R{\'{i}}o, M.~Wiebe, P.~Peterson, P.~G{\'{e}}rard-Marchant, K.~Sheppard, T.~Reddy, W.~Weckesser, H.~Abbasi, C.~Gohlke, and T.~E. Oliphant.
\newblock Array programming with {NumPy}.
\newblock \emph{Nature}, 585\penalty0 (7825):\penalty0 357--362, 2020.

\bibitem[Hinton et~al.(2015)Hinton, Vinyals, and Dean]{hinton2015distilling}
G.~Hinton, O.~Vinyals, and J.~Dean.
\newblock Distilling the knowledge in a neural network.
\newblock \emph{arXiv preprint arXiv:1503.02531}, 2015.

\bibitem[Javanmard et~al.(2024)Javanmard, Chen, Mirrokni, Badanidiyuru, and Fu]{javanmard2024loss}
A.~Javanmard, L.~Chen, V.~Mirrokni, A.~Badanidiyuru, and G.~Fu.
\newblock Learning from aggregate responses: Instance level versus bag level loss functions.
\newblock \emph{To appear in The Twelfth International Conference on Learning Representations}, 2024.

\bibitem[Kollnig et~al.(2022)Kollnig, Shuba, Van~Kleek, Binns, and Shadbolt]{kollnig2022goodbye}
K.~Kollnig, A.~Shuba, M.~Van~Kleek, R.~Binns, and N.~Shadbolt.
\newblock Goodbye tracking? {I}mpact of {iOS} app tracking transparency and privacy labels.
\newblock In \emph{Proceedings of the 2022 ACM Conference on Fairness, Accountability, and Transparency}, pages 508--520, 2022.

\bibitem[Musicant et~al.(2007)Musicant, Christensen, and Olson]{musicant2007supervised}
D.~R. Musicant, J.~M. Christensen, and J.~F. Olson.
\newblock Supervised learning by training on aggregate outputs.
\newblock In \emph{Seventh IEEE International Conference on Data Mining}, pages 252--261. IEEE, 2007.

\bibitem[Patrini et~al.(2014)Patrini, Nock, Rivera, and Caetano]{NIPS2014_a8baa565}
G.~Patrini, R.~Nock, P.~Rivera, and T.~Caetano.
\newblock (almost) no label no cry.
\newblock In \emph{Advances in Neural Information Processing Systems}, volume~27, 2014.

\bibitem[Pedregosa et~al.(2011)Pedregosa, Varoquaux, Gramfort, Michel, Thirion, Grisel, Blondel, Prettenhofer, Weiss, Dubourg, Vanderplas, Passos, Cournapeau, Brucher, Perrot, and Duchesnay]{scikit-learn}
F.~Pedregosa, G.~Varoquaux, A.~Gramfort, V.~Michel, B.~Thirion, O.~Grisel, M.~Blondel, P.~Prettenhofer, R.~Weiss, V.~Dubourg, J.~Vanderplas, A.~Passos, D.~Cournapeau, M.~Brucher, M.~Perrot, and E.~Duchesnay.
\newblock Scikit-learn: {M}achine learning in {P}ython.
\newblock \emph{Journal of Machine Learning Research}, 12:\penalty0 2825--2830, 2011.

\bibitem[Quadrianto et~al.(2008)Quadrianto, Smola, Caetano, and Le]{quadrianto2008estimating}
N.~Quadrianto, A.~J. Smola, T.~S. Caetano, and Q.~V. Le.
\newblock Estimating labels from label proportions.
\newblock In \emph{Proceedings of the 25th International Conference on Machine learning}, pages 776--783, 2008.

\bibitem[Rueping(2010)]{rueping2010svm}
S.~Rueping.
\newblock {SVM} classifier estimation from group probabilities.
\newblock In \emph{Proceedings of the 27th International Conference on Machine Learning}, pages 911--918, 2010.

\bibitem[Saket(2022)]{saket2022algorithms}
R.~Saket.
\newblock Algorithms and hardness for learning linear thresholds from label proportions.
\newblock \emph{Advances in Neural Information Processing Systems}, 35:\penalty0 1267--1279, 2022.

\bibitem[Scott and Zhang(2020)]{scott2020learning}
C.~Scott and J.~Zhang.
\newblock Learning from label proportions: {A} mutual contamination framework.
\newblock \emph{Advances in Neural Information Processing Systems}, 33:\penalty0 22256--22267, 2020.

\bibitem[Sweeney(2002)]{sweeney2002k}
L.~Sweeney.
\newblock $k$-anonymity: {A} model for protecting privacy.
\newblock \emph{International Journal of Uncertainty, Fuzziness and Knowledge-Based Systems}, 10\penalty0 (05):\penalty0 557--570, 2002.

\bibitem[Vershynin(2012)]{vershynin_2012}
R.~Vershynin.
\newblock Introduction to the non-asymptotic analysis of random matrices.
\newblock In \emph{Compressed Sensing: Theory and Applications}, page 210–268. Cambridge University Press, 2012.

\bibitem[Vershynin(2018)]{vershynin2018high}
R.~Vershynin.
\newblock \emph{High-Dimensional Probability: An Introduction with Applications in Data Science}.
\newblock Cambridge University Press, 2018.

\bibitem[Wang and Song(2011)]{wang2011ckmeans}
H.~Wang and M.~Song.
\newblock Ckmeans.1d.dp: {O}ptimal $k$-means clustering in one dimension by dynamic programming.
\newblock \emph{The R Journal}, 3\penalty0 (2):\penalty0 29, 2011.

\bibitem[Wein and Zenios(1996)]{wein1996pooled}
L.~M. Wein and S.~A. Zenios.
\newblock Pooled testing for hiv screening: capturing the dilution effect.
\newblock \emph{Operations Research}, 44\penalty0 (4):\penalty0 543--569, 1996.

\bibitem[Yu et~al.(2014)Yu, Choromanski, Kumar, Jebara, and Chang]{yu2014learning}
F.~X. Yu, K.~Choromanski, S.~Kumar, T.~Jebara, and S.-F. Chang.
\newblock On learning from label proportions.
\newblock \emph{arXiv preprint arXiv:1402.5902}, 2014.

\bibitem[Zhang et~al.(2022)Zhang, Wang, and Scott]{zhang2022learning}
J.~Zhang, Y.~Wang, and C.~Scott.
\newblock Learning from label proportions by learning with label noise.
\newblock \emph{Advances in Neural Information Processing Systems}, 35:\penalty0 26933--26942, 2022.

\end{thebibliography}

\newpage
\appendix

\section{Missing analysis from \Cref{sec:linear_regression}}
\label{app:linear_regression}

\subsection{Proof of \Cref{thm:linear-regression}}
The derivative of the loss at the minimizer is zero, which gives us:
\[ 
    \bX^\sT (\bS^\sT \bS\by - \bX\hth) = \mtx{0}.
\]
By rearranging the terms, we have
\begin{align*}
\hth - \tth &= (\bX^\sT\bX)^{-1}\bX^\sT\bS^\sT \bS\by - \tth\\
&= (\bX^\sT\bX)^{-1}\bX^\sT\bS^\sT \bS\bX\tth -\tth + (\bX^\sT\bX)^{-1}\bX^\sT\bS^\sT \bS \beps\\
&= (\bX^\sT\bX)^{-1}\bX^\sT(\bS^\sT \bS - \Iden_{n})\bX\tth+ (\bX^\sT\bX)^{-1}\bX^\sT\bS^\sT \bS \beps\,.
\end{align*}
Since the noise vector $\beps \in \R^n$ is independent of the design matrix $\bX$ with
$\E[\beps] = \mtx{0}$ and $\E[\beps\beps^\sT]=\sigma^2\Iden$, we have
\begin{align}
    \E\Big[ \twonorm{\hth-\tth}^2 ~\Big|~ \bX \Big]
    =&
    \twonorm{ (\bX^\sT\bX)^{-1}\bX^\sT(\bS^\sT \bS - \Iden_n)\bX\tth }^2 
    + \sigma^2{\rm trace} ((\bX^\sT\bX)^{-1}\bX^\sT \bS^\sT \bS \bS^\sT \bS \bX (\bX^\sT\bX)^{-1})\,.
    \label{eq:cond-exp}
\end{align}
Further, since the bags are non-overlapping,
we have $\bS\bS^\sT = \Iden_{m}$, by which we get
\begin{align}
{\rm trace} ((\bX^\sT\bX)^{-1}\bX^\sT\bS^\sT \bS \bS^\sT \bS \bX (\bX^\sT\bX)^{-1})
&= 
{\rm trace} ((\bX^\sT\bX)^{-1}\bX^\sT\bS^\sT  \bS \bX (\bX^\sT\bX)^{-1})\nonumber\\
&= \fronorm{(\bX^\sT\bX)^{-1}\bX^\sT\bS^\sT}^2\,.\label{eq:var}
\end{align}
Substituting into \eqref{eq:cond-exp}, we prove the claim.

\subsection{Proof of \Cref{cor:LinearReg}}
By definition of the operator norm, we have
\[
    \twonorm{(\bX^\sT\bX)^{-1}\bX^\sT(\bS^\sT\bS-\Iden_{n})\bX\tth}
    \le
    \opnorm{(\bX^\sT\bX)^{-1}\bX^\sT} \twonorm{(\bS^\sT\bS-\Iden_{n})\bX\tth}\,.
\]
Next, we upper bound the variance term in~\Cref{eq:var} using the inequality
\[
    \fronorm{\bA\bB}^2 \le \min\parens*{\opnorm{\bA}^2\fronorm{\bB}^2, \opnorm{\bB}^2\fronorm{\bA}^2}\,.
\]
We have $\fronorm{\bS}^2 = m$ and $\opnorm{\bS} = 1$ by the Cauchy--Schwarz inequality.
We also assumed $\text{rank}(\bX) \le d$,
which implies
\[
    \fronorm{(\bX^\sT\bX)^{-1}\bX^\sT}\le \sqrt{d} \opnorm{(\bX^\sT\bX)^{-1}\bX^\sT}\,.
\]
Therefore, we obtain
\[
    \fronorm{(\bX^\sT\bX)^{-1}\bX^\sT \bS^\sT}^2\le \opnorm{(\bX^\sT\bX)^{-1}\bX^\sT}^2 \min(m,d)\,.
\]
Combining these two bounds with \Cref{thm:linear-regression} gives the result.

\subsection{Proof of \Cref{lem:sorting}}

The key idea is to rewrite the optimization problem by ``lifting'' the space of optimization variables as follows:
\begin{align}
    &\min_{(B_1,\dots,B_m) \in \mathcal{B}} \quad \sum_{\ell=1}^m \sum_{i\in B_\ell} (\ty_i - c_\ell)^2\label{eq:1dimKmeans-alt} \\
    &~~~~\text{subject to} \quad~~ |B_\ell| \ge k \quad \forall \ell \in [m] \nonumber \\
    &~~~~\phantom{\text{subject to}} \quad~~ c_\ell \in \reals \quad \hspace{0.32cm} \forall \ell \in [m] \nonumber
\end{align}
In words, we introduce the additional variables $c_\ell \in \reals$, for $\ell\in[m]$. It is easy to see that problems~\eqref{eq:1dimKmeans} and~\eqref{eq:1dimKmeans-alt} have the same optimal bagging configurations.
Now, suppose that $\{(B^*_\ell, c^*_\ell) : \ell\in[m]\}$ is an optimal solution to~\eqref{eq:1dimKmeans-alt}.
If the claim is not true, then there exists $\ty_i > \ty_j$ and $c_\ell> c_{\ell'}$ such that $\ty_i\in B^*_{\ell'}$ and $\ty_j\in B^*_\ell$.
We then argue that by assigning $\ty_i$ to $B^*_\ell$ and $\ty_j$ to $B^*_{\ell'}$,
we can reduce the objective value,
which is a contradiction.
To show this, we must prove that
\begin{align*}
(\ty_i -c_{\ell'})^2+ (\ty_j - c_\ell)^2 > (\ty_i -c_{\ell})^2+ (\ty_j - c_{\ell'})^2
&\iff - \ty_i c_{\ell'} - \ty_j c_\ell
>-\ty_i c_\ell - \ty_j c_{\ell'}\\
&\iff (\ty_i - \ty_j) (c_\ell - c_{\ell'})>0\,,
\end{align*}
which is true by our assumption.

\section{Missing analysis from \Cref{sec:glms}}
\label{app:glms}
We recall the notion of strong convexity below.
\begin{definition}
 A function $f:\reals^n \rightarrow \reals$ is \emph{strongly convex} with parameter $\mu$ if the following holds for all $x,y\in \reals$:
 \[f(\by)\ge f(\bx) + s_{\bx}^\sT(\by-\bx) + \frac{\mu}{2}\twonorm{\by-\bx}^2\,,\]
\end{definition}
for any $s_{\bx}\in\partial f(\bx)$, where $\partial f(\bx)$ denotes the set of subgradients of $f$ at $\bx$.

The next lemma states that controlling the estimation error for a strongly convex loss function reduces to controlling the norm of the gradient of the loss at the true model.
\begin{lemma}\label{lem:taylor}
Suppose that the loss $\cL$ is strongly convex with parameter $\mu$ and $\hth = \arg\min_{\th}\cL(\th)$. Then, for any model $\tth$, we have
\[
\twonorm{\hth - \tth} \le \frac{1}{\mu}\twonorm{\cL(\tth)}\,.
\]
In addition, if $\cL$ has a Lipschitz continuous gradient with parameter $L$, we have
\[
 \frac{1}{L}\twonorm{\cL(\tth)}\le \twonorm{\hth - \tth}.
\]
\end{lemma}
\begin{proof}
By writing the definition of strong convexity for $\hth$ and $\tth$, and noting that $\nabla\cL(\hth) = \mtx{0}$, we get
\[
\cL(\tth)\ge \cL(\hth) + \frac{\mu}{2}\twonorm{\tth-\hth}^2\,.
\]
Likewise, by changing the role of $\hth$ and $\tth$, we have
\[
\cL(\hth)\ge \cL(\tth) + \nabla \cL(\tth)^\sT(\hth-\tth)+ \frac{\mu}{2}\twonorm{\tth-\hth}^2\,.
\]
By adding the above two inequalities and rearranging the terms, we arrive at
\[
\nabla \cL(\tth)^\sT(\tth-\hth)\ge \mu\twonorm{\tth-\hth}^2\,.
\]
Next, by Cauchy--Schwarz inequality, $\nabla \cL(\tth)^\sT(\tth-\hth) \le \twonorm{\cL(\tth)} \twonorm{\tth-\hth}$, which along with the previous inequality proves the first claim.

The second claim follows easily from Lipschitz condition.
We write
\[
\twonorm{\nabla\cL(\tth)} = \twonorm{\nabla\cL(\tth) - \nabla\cL(\hth)} 
\le L \twonorm{\tth - \hth}\,,
\]
which completes the proof.
\end{proof}

\subsection{Proof of \Cref{thm:GLM}}

 The gradient of the loss in~\eqref{eq:GLM_estimator} reads as
 \begin{align*}
 \nabla \cL(\th) &= 
 \frac{1}{n}\sum_{\ell=1}^m \sum_{i\in B(\ell)} \frac{1}{a_i(\phi)} (\overline{y}_\ell - b'(\th^\sT\bx_i)) \bx_i\\
 &= \bX^\sT \bD^{-1} (\bS^\sT\bS \by - b'(\bX\th))\,.
 \end{align*}
 We next consider the following bias-variance decomposition:
 \begin{align}\label{eq:decomposition}
    \E\Big[\twonorm{\nabla \cL(\tth)}^2 ~\big|~ \bX \Big]
    = 
    \twonorm{\E\Big[\nabla \cL(\tth) ~\big|~ \bX \Big]}^2
    + \trace(\Cov(\nabla \cL(\tth) \mid \bX))\,.
 \end{align}
 Under the GLM, the responses $y_i$ are independent conditioned on $\bx_i$.
 In addition, 
 \[
 \E[y_i \mid \bx_i] = b'(\th^\sT\bx_i),\quad \Var(y_i \mid \bx_i) = a_i(\phi) b''(\th^\sT\bx_i)\,.
 \]
 We therefore get
 \begin{align}
 \E\Big[\nabla \cL(\tth) ~\big|~ \bX \Big]
 &= \bX^\sT \bD^{-1} (\bS^\sT\bS b'(\bX\tth) - b'(\bX\tth))\nonumber\\
 &= \bX^\sT \bD^{-1} (\bS^\sT\bS  - \Iden) b'(\bX\tth)\,.\label{eq:E-nabla-L}
 \end{align}
 In addition,
 \begin{align*}
 \Cov(\nabla \cL(\tth) \mid \bX) &=
 \E\Big[\bX^\sT\bD^{-1}\bS^\sT\bS(\by-b'(\bX\tth))(\by-b'(\bX\tth))^\sT \bS^\sT\bS\bD^{-1} \bX ~\big|~ \bX\Big]\\
 &= \bX^\sT\bD^{-1}\bS^\sT\bS \bD\diag(b''(\bX\tth)) \bS^\sT\bS\bD^{-1} \bX\,. 
 \end{align*}
 Therefore,
 \begin{align}\label{eq:V-nabla-L}
 \trace(\Cov(\nabla \cL(\tth) \mid \bX))
 = \fronorm{\bX^\sT\bD^{-1}\bS^\sT\bS \bD^{1/2}\diag(b''(\bX\tth))^{1/2}}^2\,.
 \end{align}
 Combining~\eqref{eq:E-nabla-L} and \eqref{eq:V-nabla-L} into~\eqref{eq:decomposition} completes the proof.
 
\subsection{Proof of \Cref{cor:GLM_estimator_upper_bound}}

 We upper bound each term of~\eqref{eq:nabla-L} separately. For the first term, we have 
 \[
 \twonorm{\bX^\sT \bD^{-1}(\bS^\sT\bS  - \Iden) \bmu}\le \opnorm{\bX^\sT \bD^{-1}} \twonorm{(\bS^\sT\bS  - \Iden) \bmu}\,.
 \]
 For the second term, we develop two upper bounds and take the minimum of the two.
 
 For the first upper bound we have
 \begin{align*}
     \fronorm{\bX^\sT\bD^{-1}\bS^\sT\bS \bD^{1/2}\diag(b''(\bX\tth))^{1/2}}^2
     &\le d\opnorm{\bX^\sT\bD^{-1}\bS^\sT\bS \bD^{1/2}\diag(b''(\bX\tth))^{1/2}}^2\\
     &\le d \opnorm{\bX^\sT\bD^{-1}}^2 \opnorm{\bS^\sT\bS}^2 \opnorm{\bD^{1/2}\diag(b''(\bX\tth))^{1/2}} \\
     &\le d \opnorm{\bX^\sT\bD^{-1}}^2  \infnorm{\bv}\,.
 \end{align*}
 For the second upper bound we have
 \begin{align}
     \fronorm{\bX^\sT\bD^{-1}\bS^\sT\bS \bD^{1/2}\diag(b''(\bX\tth))^{1/2}}^2
     \le \opnorm{\bX^\sT\bD^{-1}}^2 \fronorm{\bS^\sT\bS \bD^{1/2}\diag(b''(\bX\tth))^{1/2}}^2\,,
 \end{align}
 using the inequality $\fronorm{\bA\bB}^2\le \opnorm{\bA}^2 \fronorm{\bB}^2$.

 We also note that
 \begin{align*}
      \fronorm{\bS^\sT\bS \bD^{1/2}\diag(b''(\bX\tth))^{1/2}}^2
     &=
     \trace\left(\bS^\sT \bS \bD \diag(b''(\bX\tth)) \bS^\sT\bS\right)\\
     & = \trace\left( \bD \diag(b''(\bX\tth)) \bS^\sT\bS \bS^\sT \bS\right)\\
     &=
     \trace\left(\bD \diag(b''(\bX\tth)) \bS^\sT\bS\right)\\
     & = \trace\left( \bD \diag(b''(\bX\tth)) \bS^\sT \bS\right)\\
     &= \sum_{\ell=1}^m \sum_{i\in B_\ell} \frac{a_i(\phi)b''(\bx_i^\sT\tth)}{|B_\ell|}
      = \sum_{\ell=1}^m \sum_{i\in B_\ell} \frac{v_i}{|B_\ell|}\,.
 \end{align*}
 Combining the above bounds, we obtain
 \begin{align}
     \fronorm{\bX^\sT\bD^{-1}\bS^\sT\bS \bD^{1/2}\diag(b''(\bX\tth))^{1/2}}^2
     \le \opnorm{\bX^\sT\bD^{-1}}^2 \min
     \Big(\sum_{\ell=1}^m \sum_{i\in B_\ell} \frac{v_i}{|B_\ell|}, d \infnorm{\bv}\Big)\,.
 \end{align}
 This completes the proof of the corollary.
 
\section{Missing analysis from \Cref{sec:comparison}}

\subsection{Proof of Theorem~\ref{thm:UB-risk}}
We prove the claim using the result of Corollary~\ref{cor:LinearReg}. Recall the notation $\bmu: =\bX\tth$. By Lemma~\ref{lem:sorting}, we know that the solution $\bS$ given by~\eqref{eq:1dimKmeans} has a simple sorting structure. We use that structure to construct a bagging scheme to upper bound the term $\twonorm{(\bS^\sT\bS- \Iden)\bmu}^2$. Without loss of generality, assume that $n$ is divisible by $k$. Sort the entries of $\bmu$ and construct the bags as $B(\ell) = \{(\ell-1)k+1, \dotsc, \ell k\}$ for $\ell =1, \dotsc, m:= n/k$. In addition, let $\bar{\mu}_\ell$ indicate the average of $\mu_i$'s over bag $\ell$. This construction of bags satisfy the constraint of \eqref{eq:1dimKmeans} and so we have %
\begin{align*}
    \twonorm{(\bS^\sT\bS- \Iden)\bmu}^2 &\le \sum_{\ell=1}^m \sum_{j=(\ell-1)k+1}^{\ell k}(\mu_j - \bar{\mu}_\ell)^2\\
    &\le k \sum_{\ell=1}^m (\mu_{\ell k} - \mu_{(\ell-1)k+1})^2\\
    &\le k \Big(\sum_{\ell=1}^m \mu_{\ell k} - \mu_{(\ell-1)k+1}\Big)^2\\
    &= k(\mu_{(1)} - \mu_{(n)})^2\\
    &\le 4k \infnorm{\bmu}^2.
\end{align*}
where $\bar{\mu}_i$ in the first inequality denotes the average of $\mu_i$'s over bag $i$.

We next bound $\infnorm{\bmu}^2$. Since $\bx_i$'s are $\kappa$-subgaussian, we have that $\mu_i$ is $\kappa \twonorm{\tth}$-subgaussian, for $i\in[n]$. 

\begin{lemma}\label{lem:E-max-psi}
Suppose that $\xi_1,\dotsc, \xi_n$ are centered $\eta$-subgaussian random variables. Then, with probability at least $1-\frac{1}{n}$, we have
\[
\max_{i\in[n]} \xi_i^2 \le 2 \eta^2 \log n.
\]
\end{lemma}

By using \Cref{lem:E-max-psi} we obtain 
\begin{align}\label{eq:bias-B1}
    \twonorm{(\bS^\sT\bS - \Iden)\bmu}^2\le 4k \infnorm{\bmu}^2 \le 8k\kappa^2\twonorm{\tth}^2 
    \log n \,,
\end{align}
with probability at least $1-1/n$. 
We next use the concentration bounds on the singular values of matrices with i.i.d.\ subgaussian rows to bound $\opnorm{(\bX^\sT\bX)^{-1}\bX^\sT}$. Specifically, we use \citet[Equation (5.25)]{vershynin_2012},
which states that with probability at least $1-2e^{-c_1t^2}$, the following holds true:
\begin{align}\label{eq:op-concen0}
    \opnorm{\frac{1}{n}\bX^\sT\bX - \bSigma}\le \max(\delta_1,\delta_1^2), \quad \delta_1 = C_1\sqrt{\frac{d}{n}}+\frac{t}{\sqrt{n}}\,,
\end{align}
for constants $c_1,C_1>0$ that depend only on $\kappa$. We define the probabilistic event $\event$ as follows:
\begin{align}\label{eq:event-first}
\event_1:= \left\{\opnorm{\frac{1}{n}\bX^\sT\bX - \bSigma}\le C\sqrt{\frac{d}{n}} \right\}\,,
\end{align}
for some fixed constant $C>0$.
Then using~\eqref{eq:op-concen0} we have $\prob(\event_1)\ge 1-2e^{-c_1d}$, for some constant depending on $C$ and $\kappa$.
Under the event $\event$, and by using Weyl's inequality for singular values, we have
\begin{align}\label{eq:sigma-min}
\sigma_{\min}(\bX) \ge \sqrt{n\sigma_{\min}(\bSigma) - C\sqrt{dn}}\,.
\end{align}

Combining~\eqref{eq:bias-B1} and~\eqref{eq:sigma-min}
 in Corollary~\ref{cor:LinearReg}, we obtain the result.

\begin{proof}[Proof of Lemma~\ref{lem:E-max-psi}]
Since $\xi_i$ is $\eta$-subgaussian, by definition $\E[\exp(X^2/\eta^2)] \le 2$.
Exponentiating and using Markov’s inequality, we obtain
\[
    \prob(|\xi_i|\ge t) = \prob(e^{\xi_i^2/\eta^2}\ge e^{t^2/\eta^2}) \le e^{-t^2/\eta^2}\E [e^{\xi_i^2/\eta^2}] \le 2 e^{-t^2/\eta^2}\,.
\]
Choosing $t = \eta\sqrt{2\log n}$ and union bounding over $i\in[n]$, we get
\[
\prob\Big(\max_{i\in[n]} |\xi_i|\ge \eta\sqrt{2\log n}\Big) \le \frac{2}{n}\,,
\]
which completes the proof of lemma.
\end{proof}

\subsection{Proof of Theorem~\ref{thm:LB-risk}}

We recall the characterization of the risk given by Theorem~\ref{thm:linear-regression}:
\begin{align}\label{eq:decomp2}
\E\Big[ \twonorm{\hth-\tth}^2 ~\Big|~ \bX \Big]
    =
    \twonorm{(\bX^\sT\bX)^{-1}\bX^\sT (\bS^\sT\bS - \Iden_n)\bX\tth}^2 +\sigma^2 \fronorm{(\bX^\sT\bX)^{-1}\bX^\sT\bS^\sT}^2\,.
\end{align}
We introduce the shorthand $\bLambda: = \bS^\sT\bS-\Iden_n$. Our next lemma lower bounds the expected bias.

\begin{lemma}\label{lem:bias-LB} Under the assumptions of Theorem~\ref{thm:LB-risk}, the following holds with probability at least $1- 2e^{-c_1d} -2d^{-c}$,
\[
\E\Big[\twonorm{(\bX^\sT\bX)^{-1}\bX^\sT \bLambda\bX\tth}^2 ~\Big|~ \bX\Big] \ge \left(1-\frac{1}{k} - C\frac{d\sqrt{\log d}}{\sigma_{\min}(\bSigma)\sqrt{n}}\right)^2 \twonorm{\tth}^2\,,
\]
where constants $C, c, c_1>0$ only depend on the subgaussian norm $\kappa$.
\end{lemma}

Our next lemma lower bound the variance term in~\eqref{eq:decomp2}. 
\begin{lemma}\label{lem:var-S}
Under the assumptions of Theorem~\ref{thm:LB-risk}, the following holds with probability at least $1- 2e^{-c_1d} -2d^{-c}$,
\[
\E\Big[\fronorm{(\bX^\sT\bX)^{-1}\bX^\sT \bS}^2 ~\Big|~ \bX\Big] \ge \frac{1}{kn} \cdot \frac{\trace(\bSigma) - \sqrt{d\log d}}{(\opnorm{\bSigma} + c_0\sqrt{\frac{d}{n}})^2} \,,
\]
where constants $c, c_0, c_1>0$ only depend on the subgaussian norm $\kappa$.
\end{lemma}

Proof of Theorem~\ref{thm:LB-risk} follows by using Lemma~\ref{lem:bias-LB} and Lemma~\ref{lem:var-S} in the decomposition~\eqref{eq:decomp2}.

\subsection{Proof of Lemma~\ref{lem:bias-LB}}

Consider the following optimization problem
\begin{align}
    \widehat{\balpha} = \frac{1}{2n}\arg\min_{\balpha\in\reals^d} \twonorm{\bX\balpha - \bLambda \bX\tth}^2\,.
\end{align}
It is easy to see that by the KKT condition $\widehat{\balpha} = (\bX^\sT\bX)^{-1}\bX^\sT\bLambda \bX\tth$, and so we are interested in the norm of the solution to the above optimization problem. In order to do this, we define $\balpha_*:= \frac{\trace(\bLambda)}{n}\tth$. As we will see later this is indeed the solution of the population version of the above loss (when $n\to\infty$). The strategy is to upper bound $\twonorm{\widehat{\balpha}-\balpha_*}$ from which we obtain a lower bound on $\twonorm{\widehat{\balpha}}$.

By the optimality of $\widehat{\balpha}$ we have
\begin{align*}
    0&\le \frac{1}{2n}\twonorm{\bX\balpha_* - \bLambda \bX\tth}^2 - \frac{1}{2n}\twonorm{\bX\widehat{\balpha} - \bLambda \bX\tth}^2\\
    &= \frac{1}{n}(\balpha_*-\widehat{\balpha})^\sT\bX^\sT (\bX\balpha_* - \bLambda \bX\tth) - \frac{1}{2n}\twonorm{\bX(\widehat{\balpha} - \balpha_*)}^2\,.
\end{align*}
Rearranging the terms we get
\begin{align*}
    \frac{1}{2n}\twonorm{\bX(\widehat{\balpha} - \balpha_*)}^2
    \le \twonorm{\balpha_*-\widehat{\balpha}} \frac{1}{n}\twonorm{\bX^\sT (\bX\balpha_* - \bLambda \bX\tth)}\,.
\end{align*}
The left-hand side can be also lower bounded by 
\[
\frac{1}{2}\sigma_{\min}\Big(\frac{1}{n}\bX^\sT\bX\Big) \twonorm{\widehat{\balpha}-\balpha_*}^2\le \frac{1}{2n}\twonorm{\bX(\widehat{\balpha} - \balpha_*)}^2.
\]
Combining the last two inequalities we arrive at
\begin{align}\label{eq:alpha-diff}
    \twonorm{\widehat{\balpha}-\balpha_*} \le \frac{2\twonorm{\bX^\sT (\bX\balpha_* - \bLambda \bX\tth)}}{n\sigma_{\min}(\bX^\sT\bX/n)}\,.
\end{align}
By using concentration bound on the singular values of matrices with i.i.d subgaussian rows, see~\citet[Equation (5.25)]{vershynin_2012}, we have that with probability at least $1-2e^{-c_1t^2}$,
\begin{align}\label{eq:op-concen}
    \opnorm{\frac{1}{n}\bX^\sT\bX - \bSigma}\le \max(\delta_1,\delta_1^2), \quad \delta_1 = C_1\sqrt{\frac{d}{n}}+\frac{t}{\sqrt{n}}\,,
\end{align}
for constants $c_1,C_1>0$ which depend only on $\kappa$. We define the probabilistic event $\event_1$ as follows:
\[\event_1:= \left\{\opnorm{\frac{1}{n}\bX^\sT\bX - \bSigma}\le C\sqrt{\frac{d}{n}} \right\}\,,\]
for some fixed constant $C>C_1$. Then using~\eqref{eq:op-concen} we have $\prob(\event_1)\ge 1-2e^{-c_1d}$, for some constant depending on $C$ and $\kappa$.

We next bound the numerator of the right-hand side of \eqref{eq:alpha-diff}. We write
\begin{align}\label{eq:tri}
    \frac{1}{n}\twonorm{\bX^\sT (\bX\balpha_* - \bLambda \bX \tth)} \le  \twonorm{\Big(\frac{1}{n} \bX^\sT \bX-\bSigma\Big)\balpha_*}+ 
    \twonorm{\bSigma\balpha_* - \frac{1}{n}\bX^\sT\bLambda \bX\tth}\,.
\end{align}
Under event $\event_1$ the first term is bounded by 
$C\sqrt{d/n}\twonorm{\balpha_*}$.

In addition, by its definition it is easy to see that $\bLambda$ is a projection matrix of rank $n-m$. More specifically, it projects onto the space of vectors which are zero mean on each of the $m$ bags. Therefore $\trace(\bLambda) = n-m$ and so 
\begin{align}\label{eq:norm_alpha}
\twonorm{\balpha_*} = \frac{n-m}{n}\twonorm{\tth} = (1-\frac{1}{k})\twonorm{\tth}.
\end{align}
Hence, under the event $\event_1$ the first term in~\eqref{eq:tri} is bounded by
\begin{align}\label{eq:first-term}
    \twonorm{\Big(\frac{1}{n} \bX^\sT \bX-\bSigma\Big)\balpha_*}
    \le \opnorm{\Big(\frac{1}{n} \bX^\sT \bX-\bSigma\Big)}\twonorm{\balpha_*} < C\twonorm{\tth}\sqrt{\frac{d}{n}}\,.
\end{align}

To bound the second term in the ~\eqref{eq:tri}, we note that by definition of $\balpha_*$,
\begin{align}
    \twonorm{\bSigma\balpha_* - \frac{1}{n}\bX^\sT \bLambda \bX\tth} &=
    \frac{1}{n} \twonorm{\Big(\trace(\bLambda)\bSigma - \bX^\sT \bLambda \bX\Big)\tth}\nonumber\\
    &\le \frac{\twonorm{\tth}}{n} \opnorm{\trace(\bLambda)\bSigma - \bX^\sT\bLambda\bX}\nonumber\\
    &\le \twonorm{\tth} \frac{d}{n} 
    \Big|\trace(\bLambda)\bSigma - \bX^\sT\bLambda\bX \Big|_{\infty}\,,\label{eq:second-term1}
\end{align}
where for a matrix $A$, the notation $|\bA|_\infty$ refers to the maximum absolute values of its entries. In the last step we used the inequality $\opnorm{\bA}\le d |\bA|_\infty$, for symmetric $\bA\in\reals^{d\times d}$.

We next proceed by upper bounding the right-hand side of \eqref{eq:second-term1}. We first show that the matrix of interest side has zero mean. To see this note that for any $i,j\in [d]$ we have
\[
\E[(\bX^\sT\bLambda \bX)_{ij}] = \E[\tilde{\bx}_i^\sT \bLambda \tilde{\bx}_j] = \trace(\bLambda \E(\tilde{\bx}_j\tilde{\bx}_i^\sT)) = \trace(\bLambda) \Sigma_{ij}\,,
\]
where $\tilde{\bx}_i$ denotes the $i$-th column of $\bX$. Therefore, $\E[\bX^\sT\bLambda \bX] = \trace(\bLambda) \bSigma$. We next use the (asymmetric version of) Hanson--Wright inequality~(see, e.g, \citet[Theorem 6.2.1]{vershynin2018high}), by which we get that for any fixed $i,j\in [d]$,
\begin{align}
    \prob\left\{|(\bX^\sT\bLambda \bX)_{ij} - \trace(\bLambda) \Sigma_{ij}|\ge t\right\}\le 2\exp\left\{-c\min\Big( \frac{t^2}{\kappa^2n(1-1/k)}, \frac{t}{\kappa^2}\Big)\right\}\,,
\end{align}
where we used the fact that $\fronorm{\bLambda} = n-m = n - n/k$, since it is a projection matrix of rank $n-m$. By union bounding over the $d^2$ coordinates $i,j\in[d]$, we get
\begin{align}\label{eq:HW}
    \prob\left\{|\bX^\sT\bLambda \bX - \trace(\bLambda) \bSigma|_{\infty}\ge t\right\}\le 2d^2\exp\left\{-c_0\min\Big( \frac{t^2}{\kappa^2n(1-1/k)}, \frac{t}{\kappa^2}\Big)\right\}\,.
\end{align}
Fix a constant $C>\sqrt{\frac{2}{c_0}}\kappa$ and define the event $\event_2$ as follows
\begin{align*}
    \event_2:= \left\{|\bX^\sT\bLambda \bX - \trace(\bLambda) \bSigma|_{\infty} \le C\sqrt{n\log d}\right\}\,.
\end{align*}
Using the deviation bound~\eqref{eq:HW} we have $\prob(\event_2)\ge 1-2d^{-c}$ with $c= \frac{C^2c_0}{\kappa^2(1-1/k)} - 2>0$.
Recalling the bound~\eqref{eq:second-term1}, on the event $\event_2$ we have
\begin{align}\label{eq:second-term2}
   \twonorm{\bSigma\balpha_* - \frac{1}{n}\bX^\sT \bLambda \bX\tth} \le C\twonorm{\tth} d\sqrt{\frac{\log d}{n}}\,.
\end{align}
Putting together equations~\eqref{eq:tri}, \eqref{eq:first-term}, \eqref{eq:second-term2},
we obtain that on the event $\event:= \event_1\cap \event_2$,
\begin{align}\label{eq:numerator}
\frac{1}{n} \twonorm{\bX^\sT (\bX\balpha_* - \Lambda \bX\tth)} \le C\twonorm{\tth} d\sqrt{\frac{\log d}{n}}\,,
\end{align}
for a constant $C$ depending on the subgaussian norm $\kappa$. In addition on the event $\event_1$, we have
\begin{align}\label{eq:sigma-min}
\sigma_{\min}\Big(\frac{1}{n}\bX^\sT\bX\Big) \ge \sigma_{\min}(\bSigma) - \opnorm{\frac{1}{n}\bX^\sT\bX - \bSigma}\ge \sigma_{\min}(\bSigma) - C\sqrt{\frac{d}{n}}\,.
\end{align}
Next by combining~\eqref{eq:sigma-min} and \eqref{eq:numerator} into~\eqref{eq:alpha-diff}, we get that 
\begin{align}
\twonorm{\balpha_*-\widehat{\balpha}}\le \frac{Cd}{\sigma_{\min}(\bSigma)}\sqrt{\frac{\log d}{n}}\twonorm{\tth},
\end{align}
for some constant $C>0$. Note that here we used the fact that $d= o(n)$. Therefore, by using triangle inequality, on the event $\event$
\[
\twonorm{\widehat{\balpha}} \ge \twonorm{\balpha_*} - \twonorm{\widehat{\balpha}-\balpha_*}
\ge \left(1-\frac{1}{k} - C\frac{d\sqrt{\log d}}{\sigma_{\min}(\bSigma)\sqrt{n}}\right) \twonorm{\tth}\,,
\]
for a constant $C>0$ that depends on the subgaussian norm $\kappa$. 

We also have
\[
\prob(\event) = 1- \prob(\event_1^c\cup \event_2^c) \ge 1- \prob(\event_1^c)- \prob(\event_2^c) 
\ge 1- 2 e^{-c_1d}- 2d^{-c}\,, 
\]
which along with the previous equation gives the desired result.
\subsection{Proof of Lemma~\ref{lem:var-S}}

Write $\bS^\sT = [\bs_1|\dotsc|\bs_m]$ with $\bs_i\in \reals^{n}$ and $\twonorm{\bs_i} = 1$. We then have 
\begin{align}\label{eq:sum_s}
\fronorm{(\bX^\sT\bX)^{-1}\bX^\sT \bS}^2 = \sum_{i=1}^n \twonorm{(\bX^\sT\bX)^{-1}\bX^\sT \bs_i}^2.
\end{align}
We next show that for any unit vector $\bs$ which is independent of data $(\by,\bX)$ we have
\begin{align}\label{eq:var-inter1}
\twonorm{(\bX^\sT\bX)^{-1}\bX^\sT \bs}^2 \ge \frac{\trace(\bSigma) - \sqrt{d\log d}}{n^2(\opnorm{\bSigma} + c_0\sqrt{\frac{d}{n}})^2} \left(1- 2e^{-c_1d} -2d^{-c}\right)\,,
\end{align}
which together with~\eqref{eq:sum_s} and the fact that $m=n/k$, implies the claim of Lemma~\ref{lem:var-S}. 

Define $\bv:= (\bX^\sT\bX)^{-1}\bX^\sT \bs$. Therefore, $\frac{1}{n}\bX^\sT\bX\bv = \frac{1}{n} \bX^\sT \bs$, which implies that
\begin{align}\label{eq:stationary}
\opnorm{\frac{1}{n}\bX^\sT\bX}^2 \twonorm{\bv}^2\ge \twonorm{\frac{1}{n}\bX^\sT\bs}^2\,.
\end{align}
Our strategy to lower bound $\E[\twonorm{\bv}^2]$ is to upper bound the left-hand side of~\eqref{eq:stationary}  and lower bound it right-hand side.

For the first task, recall the concentration bound~\eqref{eq:op-concen}. By taking $t = c'\sqrt{d}$ in that bound, we obtain
\begin{align}\label{eq:op-norm1}
    \prob\left(\opnorm{\frac{1}{n}\bX^\sT\bX - \bSigma} \le c_0\sqrt{\frac{d}{n}} \right) \ge 1-2e^{-c_1d},
\end{align}
for some constants $c_0,c_1$ depending on $\kappa$, the subgaussian norm of rows of $\bX$.
We refer to the probabilistic event in~\eqref{eq:op-norm1} by $\event_1$.

We then proceed to the second task, i.e., lower bounding $\twonorm{\frac{1}{n}\bX^\sT\bs}$. To do this, denote the columns of $\bX\in\reals^{n\times d}$ by $\tilde{\bx}_1,\dotsc, \tilde{\bx}_d\in\reals^{n}$. In this notation, 
\begin{align}\label{eq:sum}
\twonorm{\bX^\sT\bs}^2 = \sum_{\ell=1}^d (\tilde{\bx}_\ell^\sT \bs)^2 := \sum_{\ell=1}^d Z_\ell^2\,.
\end{align}
By assumption, $Z_\ell = \tilde{\bx}_\ell^\sT \bs$ are independent subgaussian random variables with $\E[Z_\ell^2] = \Sigma_{\ell,\ell}$ and the subgaussian norm $\|Z_\ell\|_{\psi_2}\le C\kappa$ for a universal constant $C>0$. Therefore, by~\citet[Remark 5.18 and Lemma 5.14]{vershynin_2012}, $Z_\ell^2 - \Sigma_{\ell,\ell}$  are independent centered sub-exponential random
variables with $\|Z_\ell^2 - \Sigma_{\ell,\ell}\|_{\psi_1}\le 2\|Z_\ell^2\|_{\psi_1}\le 4\|Z_{\ell}\|_{\psi_2}^2\le 4C^2\kappa^2:= C_0$. Here, $\|\cdot\|_{\psi_1}$ refers to the subexponential norm of a random variable. We can therefore use an exponential deviation inequality, \citet[Corollary 5.17]{vershynin_2012} to control sum \eqref{eq:sum}. This gives us for every $\eps\ge 0$,
\[
    \prob\left(\Big|\twonorm{\bX^\sT\bs}^2-\trace(\bSigma) ~\Big|~ \ge \eps d\right)
    =
    \prob\left(\Big|\sum_{\ell=1}^d Z_\ell^2 - \trace(\bSigma) ~\Big|~ \ge \eps d\right)
    \le
    2 \exp\Big[-c\min\Big(\frac{\eps^2}{C_0^2}, \frac{\eps}{C_0}\Big)d\Big]\,,
\]
where $c>0$ is an absolute constant. We take $\eps = \sqrt{(\log d)/d}$ and define the probabilistic event  
\[
\event_2: = \left\{ \Big|\twonorm{\bX^\sT\bs}^2-\trace(\bSigma) \Big| \le \sqrt{d\log d}\right\}\,.
\]
By the above deviation bound we have $\prob(\event_2)\ge 1- 2d^{-c}$ for some constant $c>0$.

We next consider the event $\event:=\event_1\cap\event_2$. Using~\eqref{eq:op-norm1} and the above bound on $\prob(\event_2)$ we get
\[
\prob(\event) =1 - \prob(\event_1^c\cup\event_2^c)
\ge 1- \prob(\event_1^c) - \prob(\event_2^c)
\ge 1- 2e^{-c_1d} -2d^{-c}\,.
\]
Further, on the event $\event$ we have
\begin{align}
    &\opnorm{\frac{1}{n}\bX^\sT\bX} \le \opnorm{\bSigma} + \opnorm{\frac{1}{n}\bX^\sT\bX -\bSigma}\le \opnorm{\bSigma} + c_0\sqrt{\frac{d}{n}}\,.\\
    &\twonorm{\bX^\sT\bs}^2 \ge \trace(\bSigma) - \Big| \twonorm{\bX^\sT\bs}^2-\trace(\bSigma)\Big| \ge  \trace(\bSigma) - \sqrt{d\log d}\,.
\end{align}
Therefore, by invoking~\eqref{eq:stationary}, on the event $\event$ we have
\begin{align}
    \twonorm{\bv}^2 \ge \frac{\twonorm{\frac{1}{n}\bX^\sT \bs}^2}{\opnorm{\frac{1}{n}\bX^\sT\bX}^2}\ge
    \frac{\trace(\bSigma) - \sqrt{d\log d}}{n^2(\opnorm{\bSigma} + c_0\sqrt{\frac{d}{n}})^2}\,.
\end{align}
Since $\twonorm{\bv}^2$ is non-negative by an application of Markov's inequality we get
\[
\E[\twonorm{\bv}^2] \ge \frac{\trace(\bSigma) - \sqrt{d\log d}}{(\opnorm{\bSigma} + c_0\sqrt{\frac{d}{n}})^2} \prob(\event) \ge 
\frac{\trace(\bSigma) - \sqrt{d\log d}}{n^2(\opnorm{\bSigma} + c_0\sqrt{\frac{d}{n}})^2} \left(1- 2e^{-c_1d} -2d^{-c}\right)\,.
\]
This completes the proof of~\eqref{eq:var-inter1} and concludes the proof of Lemma~\ref{lem:var-S}.

\subsection{Proof of Theorem~\ref{thm:UB-risk-approximate}}
The proof is similar to the proof of Theorem~\ref{thm:UB-risk}. We consider the bias-variance decomposition of the upper bound given in  Corollary~\eqref{cor:LinearReg}. 

We have 
\begin{align}\label{eq:app-step1}
    \twonorm{(\tilde{\bS}^\sT\tilde{\bS}-\Iden_n)\bX\tth}^2 \le 2\twonorm{(\bS^\sT\bS-\Iden_n)\bX\tth}^2 + 2\twonorm{(\bS^\sT\bS-\tilde{\bS}^\sT\tilde{\bS})\bX\tth}^2\,.
\end{align}
The first term is bounded in Theorem~\ref{thm:UB-risk}. For the second term, we bound it as
\begin{align}
\twonorm{(\bS^\sT\bS-\tilde{\bS}^\sT\tilde{\bS})\bX\tth}^2 &\le \opnorm{\bS^\sT\bS-\tilde{\bS}^\sT\tilde{\bS}}^2\opnorm{\bX}^2\twonorm{\tth}^2\nonumber\\
&\le \eps^2 \twonorm{\tth}^2 \sigma_{\max}(\bX^\sT\bX)\,.\label{eq:app-step2}
\end{align}
Combining~\eqref{eq:app-step1}  and~\eqref{eq:app-step2} into Corollary~\ref{cor:LinearReg}, we see that the mismatch between bagging configuration $\bS$ and $\tilde{\bS}$ contributes an inflation term to the model risk which is upper bounded by 
\begin{align*}
    &\opnorm{(\bX^\sT\bX)^{-1}\bX^\sT}^2 \twonorm{(\bS\bS^\sT - \tilde{\bS}\tilde{\bS}^\sT)\bX\tth}^2\\
    &\le \sigma_{\max}((\bX^\sT\bX)^{-1})\;\eps^2 \twonorm{\tth}^2 \sigma_{\max}(\bX^\sT\bX)\\
    &= \frac{\sigma_{\max}\left(\frac{1}{n}\bX^\sT\bX\right)}{\sigma_{\min}\left(\frac{1}{n}\bX^\sT\bX\right)}\eps^2 \twonorm{\tth}^2\,.
\end{align*}

Note that the result of Theorem~\ref{thm:UB-risk} is under an event with probability at least $1-1/n-2e^{-cd}$. Under this same event, we have $\opnorm{\frac{1}{n}\bX^\sT\bX -\bSigma}\le C\sqrt{\frac{d}{n}}$ (see Equation~\eqref{eq:event-first}). Therefore, by Weyl's inequality for singular values we have
\[
\sigma_{\max}\left(\frac{1}{n}\bX^\sT\bX\right) \le \sigma_{\max}(\bSigma) + C\sqrt{\frac{d}{n}}\,,\quad
\sigma_{\min}\left(\frac{1}{n}\bX^\sT\bX\right) \ge \sigma_{\min}(\bSigma) - C\sqrt{\frac{d}{n}}\,,
\]
which completes the proof of theorem.

\section{Missing analysis from \Cref{sec:algorithm}}
\label{app:algorithm}

\subsection{Proof of Lemma~\ref{lem:fast_k_means_solve}}

We build on the observation in \Cref{lem:sorting} about the sorted structure of an optimal solution.
First, sort the points by their~$\ty_i$ value in $O(n \log n)$ time.
Next, we present a dynamic programming algorithm that optimally slices the sorted list,
i.e., a stars-and-bars partition where each part has size at least $k$.

Define the function $f_k(i)$ to be the objective of an optimal solution for the subproblem defined by the first $i$ points.
It follows that
\[
    f_k(i)
    = 
    \begin{cases}
        \infty & \text{if $i < 0$} \\
        0 & \text{if $i = 0$} \\
        \min_{k \le s \le i} \set*{f_k(i - s) + \sum_{j=i-s+1}^i \parens*{\ty_j - \mu_{i,s}}^2} & \text{if $i \ge 1$}
    \end{cases}
\]
where
\[
    \mu_{i,s} = \frac{1}{s} \sum_{j = i - s + 1}^i \ty_j.
\]
This recurrence considers all suffixes of size $s \ge k$ as the last cluster, computes their sum of squares error,
and recursively solves the subproblem on the remaining points via $f_k(i-s)$.
This naively leads to an $O(n^3)$-time dynamic programming algorithm.
However, there are two observations that allow us to reduce the running time to $O(nk)$:
\begin{enumerate}
    \item We can assume each cluster in an optimal solution has size $k \le s < 2k$.
    If not, we can split a cluster of size $s \ge 2k$ into two parts without increasing the objective.
    It follows that we can compute each $f_k(i)$ by considering $O(k)$ recursive states.
    
    \item We can iteratively compute the sum of squared errors $d(i,s) := \sum_{j=i-s+1}^i (\ty_j - \mu_{i,s})^2$
    in constant time, as shown in \citet{wang2011ckmeans}:
    \begin{align*}
        d(i,s) &= d(i,s-1) + \frac{s-1}{s} \parens*{\ty_{i-s+1} - \mu_{i,s-1}}^2 \\
        \mu_{i,s} &= \frac{\ty_{i-s+1} + (s-1)\mu_{i,s-1}}{s}
    \end{align*}
    This means each value of $f_k(i)$ can be computed in $O(k)$ time.
\end{enumerate}
Putting everything together, we can compute $f_k(n)$ and reconstruct an optimal clustering in $O(nk)$ time after sorting.

\end{document}